\documentclass{article}

\usepackage[preprint]{neurips_2026}

\usepackage{microtype}
\usepackage{graphicx}
\usepackage{subcaption}
\usepackage{booktabs} 
\usepackage{hyperref}
\usepackage{url}
\usepackage{amsfonts}       
\usepackage{multirow}
\usepackage{array} 
\usepackage{amsmath}
\usepackage{algorithm}
\usepackage{algorithmic}
\usepackage{amssymb}
\newcommand{\PM}{PEML}
\usepackage{wrapfig}
\usepackage{caption}  
\usepackage{array} 
\usepackage{amsthm}

\usepackage{enumitem}

\newtheorem{theorem}{Theorem}[section]

\newtheorem{proposition}[theorem]{Proposition}
\newtheorem{assumption}{Assumption}

\usepackage[utf8]{inputenc} 
\usepackage[T1]{fontenc}    
\usepackage{nicefrac}       
\usepackage{xcolor}         

\title{PEML: Parameter-efficient Multi-Task Learning with Optimized Continuous Prompts}

\author{%
  Anjir Ahmed Chowdhury\thanks{Use footnote for providing further information
    about author (webpage, alternative address)---\emph{not} for acknowledging
    funding agencies.} \\
  Department of Computer Science\\
  University of Houston\\
  \texttt{aachowd4@cougarnet.uh.edu} \\
  \and
  Syed Zawad \\
  IBM Research \\
  \texttt{szawad@ibm.com} \\
  \AND
  Xiaolong Ma \\
  Argonne National Laboratory \\
  \texttt{xma@anl.gov} \\
  \and
  Xu Dong \\
  Department of Computer Science\\
  University of Houston\\
  \texttt{xdong24@cougarnet.uh.edu} \\
  \and
  Feng Yan \\
  Department of Computer Science\\
  University of Houston \\
  \texttt{fyan5@central.uh.edu} \\
}

\begin{document}

\maketitle

\begin{abstract}
Parameter-Efficient Fine-Tuning (PEFT) is widely used for adapting Large Language Models (LLMs) for various tasks. 
Recently, there has been an increasing demand for fine-tuning a single LLM for multiple tasks because it requires overall less data for fine-tuning thanks to the common features shared among tasks. More importantly, LLMs are resource demanding and deploying a single model for multiple tasks facilitates resource consolidation and consumes significantly less resources compared to deploying individual large model for each task. 
Existing PEFT methods like LoRA and Prefix Tuning are designed to adapt LLMs to a specific task.
LoRA and its variation focus on aligning the model itself for tasks, overlooking the importance of prompt tuning in multi-task learning while Prefix Tuning only adopts a simple architecture to optimize prompts, which limits the adaption capabilities for multi-task. 
To enable efficient fine-tuning for multi-task learning, it is important to co-optimize prompt optimization and model adaptation. 
In this work, we propose a Parameter-Efficient Multi-task Learning (\PM), which employs a neural architecture engineering method for optimizing the continuous prompts while also performing low-rank adaption for model weights.   
We prototype \PM\ by creating an automated framework for optimizing the continuous prompts and adapting model weights. 
We evaluate \PM\ against state-of-the-arts multi-task learning methods MTL-LoRA, MultiLoRa, C-Poly, and MoE, on the GLUE, SuperGLUE, Massive Multitask Language Understanding, and commonsense reasoning benchmarks. The evaluation results present an average accuracy improvement of up to 6.67\%, with individual tasks showing peak gains of up to 10.75\%.
\end{abstract}

\section{Introduction}

Large language models (LLMs) have made significant advancements in various natural language processing tasks such as machine translation \cite{Lewis2019}, text generation \cite{Chung2022}, and code analysis \cite{Wang2021, Qin2024}. Traditional task-specific fine-tuning (FT) becomes increasingly computationally expensive as LLMs continue to grow in size. It requires adjusting all of the model's parameters, which makes it difficult to scale \cite{Devlin2018, Howard2018, Colin2020}, and thus motivated parameter-efficient fine-tuning (PEFT) methods that only require learning a small set of additional parameters for each task \cite{Neil2019, Brian2021}. These methods \cite{Pfeiffer2020, Edward2021} have been widely adopted as they offer comparable performance to full fine-tuning while significantly reducing computational overhead \cite{Neil2019, Brian2021, Ding2023}.

\begin{wrapfigure}{r}{0.35\textwidth}
    \centering
    \scriptsize
    \vspace{-5pt}
    \includegraphics[width=0.35\textwidth]{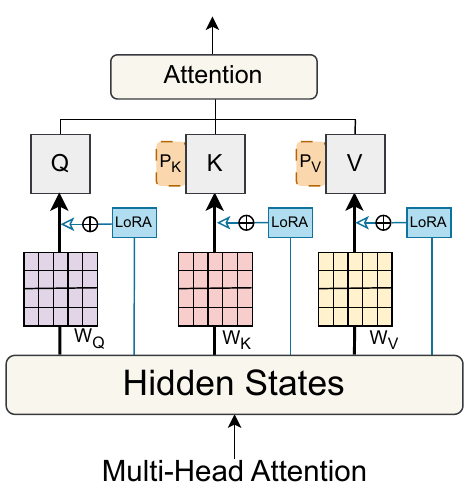}
    \caption{Overview of \PM. Hidden states are projected into queries, keys, and values using $W_Q$, $W_K$ and $W_V$, with LoRA applied to key and value projections. Learnable prefix vectors $P_K$ and $P_V$ are prepended to the key and value sequences, enabling task-specific conditioning during multi-head attention.}\
    \vspace{-10pt}
    \label{fig2}
\end{wrapfigure}


LoRA \cite{Edward2021} and Prefix Tuning \cite{Xiang2021} are among prominent PEFT methods for adapting models to a single task. LoRA enhances efficiency by introducing trainable low-rank matrices into a subset of the model's weights during training, enabling learning directional updates in the parameter space. Prefix Tuning improves adaptability by generating task-specific continuous vectors \cite{Pengfei2021, Xiang2021} to the input embeddings before each transformer layer to steer the model's generation process without modifying the original model parameters. Such learned prefix vectors and low-rank matrices enable efficient adaptation to new and related tasks with minimal additional training \cite{Brian2021, Xiao2022, Xiao2021}.

LoRA and Prefix Tuning, however, face challenges when applied to multi-task training. 
First, deploying many task-specific adapters (e.g., prefix vectors or LoRA matrices) increases memory usage and makes resource management complex.
Frequent switching between adapters incurs computational costs due to the need for adapter loading and model reconfiguration. 
Therefore, it is inefficient and costly for inference serving deployment.
In addition, individual task training prevents knowledge sharing across tasks, missing opportunities to leverage insights from one task to improve others \cite{lopes2023, Amir2020}. Such isolated task training limits potential performance gains from inter-task knowledge-sharing. 

Lately, there are efforts to adapt PEFT methods for multi-task learning. MPT \cite{Zhen2023} learns a shared transferable prompt distilled from multiple task-specific prompts and applies multiplicative low-rank adaptations for downstream task specialization. However, it requires pre-training individual teacher prompts for each task. MultiLoRA \cite{wang2023Multilora} extends LoRA by horizontally scaling modules, dividing them into parallel sub-modules with separate scaling factors. Yet, this approach increases VRAM usage due to activation caching for multiple parallel modules. C-Poly \cite{wang2023C} employs a skill-based framework that merges shared and task-specific low-rank parameters using a learned skill matrix, but its fixed architecture limits generalization. 
MTL-LoRA \cite{yang2025} introduces task-adaptive parameters that reduce interference in shared low-dimensional spaces. However, it requires task-specific routing during inference, which complicates inference deployment and resource management. Despite these advancements, most approaches focus on extending LoRA but overlook the critical aspect of prompt alignment in multi-task environment \cite{shen2024, xin2024}. Aligned prompts can significantly improve model generalization during multi-task \cite{xu2022match} training. Motivated by this observation, we explore integrating prompt alignment into PEFT methods to enhance multi-task performance.


To this end, we propose PrefixNAS which 
generates and optimizes a single, unified continuous prompt architecture through neural architecture search (NAS) for better alignment of the model’s behavior in multi-task learning. PrefixNAS captures task-relevant features and relationships, allowing the prompt encoder to leverage shared knowledge efficiently while preserving task-specific distinctions ($\sim$ see Appendix \ref{lowdata}). Additionally, PrefixNAS automatically tunes both the prefix architecture and its hyperparameters, eliminating the need for manual adjustments when adapting to new tasks. 
We further develop a Parameter-efficient Multi-Task Learning framework (\PM) to integrate PrefixNAS enabled prompt optimization into LoRA for model alignment. 
Figure \ref{fig2} shows the structure of \PM. LoRA matrices are applied to all projection layers, while prefix vectors are added in parallel to only the key and value projections of each attention head. This combination allows the model to adapt to new tasks, with LoRA handling model adaptation and prefix vectors handling input alignment. Such an integrated design also facilitates efficient inference deployment as only one adapter needs to be deployed and does not require adapter switching.

We formulate multi-task learning as a joint optimization problem through LoRA and PrefixNAS and conduct theoretical analysis on \PM.
We evaluate \PM\ by comparing it with state-of-the-arts multi-task learning methods such as MTL-LoRA, MultiLoRa, C-Poly, and MoE~\cite{yang2025, Noam2017, wang2023Multilora, wang2023C}, on the GLUE\cite{Alex2018}, SuperGLUE\cite{Alex2019}, Massive Multitask Language Understanding \cite{hendryckstest2021} and commonsense reasoning benchmarks \cite{bisk2020piqa,sakaguchi2020winogrande,mihaylov2018openbookqa,zellers2019hellaswag,clark2018think}. 
The evaluation results demonstrate an average accuracy improvement of up to 6.67\%, with individual tasks showing peak gains of up to 10.75\%.

\section{Related Work}
Approaches to parameter-efficient fine-tuning can be broadly classified into three categories: adapter-based methods, prompt-based methods (e.g., Prefix Tuning), and low-rank adaptation methods, with most recent developments primarily focusing on Prefix Tuning and LoRA.

\textbf{Adapter-based methods} \cite{Neil2019,Junxian2021,Rabeeh2021} insert small, trainable modules into a pretrained model while keeping the rest of the model frozen, capturing task-specific information with minimal added parameters. It introduce additional layers, leading to parameter redundancy, whereas LoRA focuses on low-rank updates without introducing additional layers

\textbf{Prefix Tuning} \cite{Xiang2021} is a specialized form of prompt-based fine-tuning that focuses on prepending learnable continuous vectors, known as "prefixes," to the inputs. It involves the optimization of continuous vectors that shift the model towards specific downstream tasks. Prefix Tuning updates only the prefixes during fine-tuning, keeping the base model parameters frozen. This makes it significantly more memory efficient and scalable, especially for large-scale models. However, Prefix tuning remains sensitive to initialization, which may limit its adaptability in multitask settings.

\textbf{LoRA} \cite{Edward2021} reduces the number of trainable parameters by applying low-rank decomposition to simulate weight updates in frozen models, enabling efficient fine-tuning without increasing inference costs. Several variants have been proposed to further enhance its efficiency and applicability. \textbf{AdaLoRA} \cite{zhang2023} leverages singular value decomposition (SVD) to prune less significant components. while \textbf{rsLoRA} \cite{kalajdzievski2023} introduces a scaling factor to stabilize the rank. \textbf{DoRA} \cite{liu2024} implements dynamic optimization of LoRA parameters during training to improve adaptability across learning tasks. In the context of Stable Diffusion, \textbf{Yeh et al.} \cite{yeh2024} proposed a unified LoRA framework that applies different combinations of LoRA methods for various tasks. \textbf{VeRA} \cite{kopiczko2024} introduces scaling vectors that adjust pairs of frozen random matrices shared across layers, further optimizing parameter efficiency. 

\textbf{Multi-task learning (MTL)} trains models to solve multiple related tasks simultaneously by sharing parameters across tasks \cite{Yu2020, Ruder2018}. It often involves fine-tuning on several tasks before transferring knowledge to a new one \cite{Vu2020, Colin2020, Aghajanyan2021}. Building on the foundational PEFT techniques, recent innovations have proposed MTL-specific adaptations designed to minimize interference between tasks while maintaining parameter efficiency. One approach, \textbf{MPT} \cite{Zhen2023}, learns a shared transferable prompt distilled from multiple task-specific prompts. However, its major drawback is the need to pre-train individual teacher prompts for each source task before distilling knowledge into a shared prompt, which introduces significant computational overhead. \textbf{UniPELT} \cite{mao2022unipelt} integrates Prefix-Tuning, LoRA, and adapters within a single framework and uses a gating mechanism to select among these modules. However, it is primarily designed for single-task and it performs well in low-data settings but its gains reduce under full-data training. \textbf{MTL-LoRA} \cite{yang2025} extends the original LoRA framework by introducing task-adaptive parameters that preserve task-specific information and reduce interference in shared low-dimensional spaces, enhancing multi-task adaptation. Unlike standard LoRA, which merges adapters into the base model, MTL-LoRA requires task-specific routing during inference, resulting in added latency. \textbf{MultiLoRA} \cite{wang2023Multilora} addresses the limitations of LoRA’s reliance on top singular vectors by horizontally scaling LoRA modules and diversifying their initialization, resulting in more balanced and effective adaptation across diverse tasks. However, it introduces a linear increase in VRAM usage during training due to the activation caching required for multiple parallel LoRA modules. \textbf{Customized Polytropon (C-Poly)} \cite{wang2023C} is a modular, skill-based framework that enhances multi-task learning by combining shared and task-specific low-rank parameters through a learned skill assignment matrix. Theoretically, It struggles with unseen tasks, relying on it's fixed architecture with static skill modules limits generalization. Furthermore, existing approaches lack a mechanism to automatically adapt their architecture or hyperparameters to new data, requiring manual adjustments for each benchmarks. 

Additional related work on adapter and prompt-based methods, more recent LoRA and Prefix Tuning variants, and hypernetwork-based approaches is provided in Appendix~\ref{sec:additional_related}.

\section{\PM}
\PM\ strategically handle the challenges of prompt alignment and low-rank model adaptability in multi-task learning by integrating standard LoRA with PrefixNAS. Unlike existing approaches that extend LoRA (such as MTL-LoRA, MultiLoRA) without considering for prompt alignment, \PM\ introduces a more cohesive framework where PrefixNAS dynamically adjusts the prefix structure based on the specific requirements. This adaptability allows \PM\ to maintain a shared prefix module throughout training and inference, reducing the reliance on static architecture or any pre-trained teacher prefixes. Meanwhile, LoRA operates in parallel during training to effectively optimize low-rank adaption of the model without introducing a linear increase in VRAM usage. After training, LoRA is merged with the base model, resulting in a leaner architecture that only retains the adaptive PrefixNAS module. PrefixNAS is allowed to continuously refine its structure based on evolving task demands and ensuring optimal prompt alignment and model generalization.

\begin{figure*}[!ht]
\centering
    \includegraphics[width=0.99\textwidth]{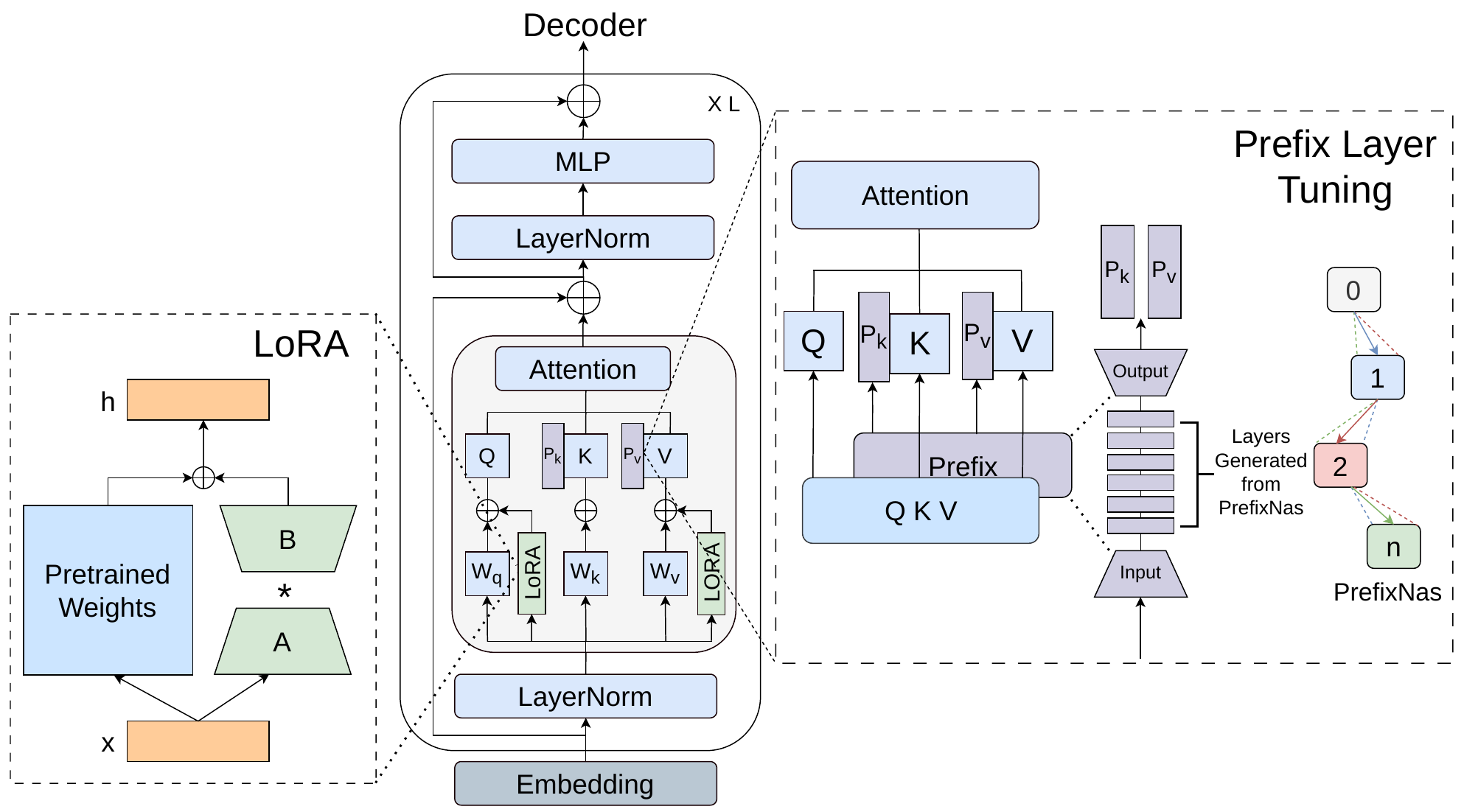}
    \caption{Unified view of \PM\ method. The left side illustrates LoRA, where matrix B is set to 0 and matrix A follows a normal distribution N(0, $\sigma^2$). The right side represents PrefixNAS, showcasing the optimal architecture derived from the search process.}
    \label{fig1}
\end{figure*}

\begin{minipage}[t]{0.52\textwidth}

\PM\ combines LoRA with PrefixNAS which use a gradient-based NAS approach that’s built on continuous relaxation techniques. Both techniques optimize parallelly during training, as shown in Figure \ref{fig1}. This allows the model to adjust shared features and dynamically generate unified prefix at the same time. Moreover, this interaction can introduce instability; we therefore provide a convergence analysis Appendix \ref{ca} showing that the joint optimization is well-behaved and convergent. Also a detailed breakdown of how the framework works is provided in Algorithm \ref{algo1}. In Algorithm \ref{algo1} the base model $\Phi$ is initialized with frozen pre-trained weights, while LoRA parameters $B, A$ and PrefixNAS architecture parameters $\alpha$ are trained concurrently. In each iteration, PrefixNAS generates candidate prefix architectures $\mathcal{A}_i(\alpha)$, which are used to construct shared prefix $\mathbf{P}_i$. That unified prefix are concatenated with inputs and passed through the model adapted by LoRA parameters, producing predictions and computing a combined loss and architecture regularization terms. PrefixNAS parameters update via gradient descent, followed by pruning weak architectures. The final model $\Phi_{\text{final}}$ consists of the optimized LoRA weights and a unified Prefix architecture.
\end{minipage}
\hfill
\begin{minipage}[t]{0.46\textwidth}
\vspace{-17pt}
\begin{algorithm}[H]
\caption{\PM: Joint Optimization with LoRA \& PrefixNAS}
\label{algo1}
\scriptsize
\begin{algorithmic}
\footnotesize
\STATE $\Phi \gets \text{pre-trained model}$ \\
\STATE $\theta \gets \text{base weights [Frozen]}$ \\
\STATE $\{D_i\}_{i=1}^{n} \gets \text{datasets}$ \\
\STATE $h \gets \text{PrefixNAS search space}$ \\
\STATE $s \gets \text{PrefixNAS operations}$ \\

\STATE Initialize LoRA params $B,A$ and PrefixNAS params $\alpha$ \\
\STATE \textbf{Parallel Joint Optimization} \\
\FOR{joint iteration $t=1$ \textbf{to} $T$}
    \STATE Generate candidate prefix architectures: $\mathcal{A}_i(\alpha) \gets \text{PrefixNAS}(s)$ \\
    \STATE $\mathbf{P}_i \gets \text{SharedPrefix}(\mathcal{A}_i(\alpha), h)$ \COMMENT{Unified prefix}
    
    \FOR{each dataset $D_i$}
        \STATE $\tilde{X}_i = \mathbf{P}_i \oplus X_i$ \COMMENT{Input with a unified prepend prefix}
        \STATE $\theta' \gets \theta + B \cdot A$ \COMMENT{LoRA adaptation}
        
        \STATE Forward pass: $\hat{y} = \Phi(\tilde{X}_i; \theta')$ \\
        \STATE Compute loss: $\mathcal{L}_t = \mathcal{L}(\hat{y}, y_i) + \lambda\mathcal{R}(\alpha)$ \\
        
        \STATE \textbf{Simultaneous Updates:}
        \STATE $\triangleright$ Update LoRA params: $B,A \leftarrow B,A - \eta\nabla_{B,A}\mathcal{L}_t$ \\
        \STATE $\triangleright$ Update NAS params: $\alpha \leftarrow \alpha - \eta\nabla_{\alpha}\mathcal{L}_t$ \\
    \ENDFOR
    
    \STATE Prune weak architectures via PrefixNAS \\
\ENDFOR
\STATE \textbf{return} $\Phi_{\text{final}}(\theta', \alpha^*)$ \COMMENT{Jointly optimized model}
\end{algorithmic}
\end{algorithm}

\end{minipage}

\subsection{Problem Statement}
Let $\mathcal{D} = \{D_1, D_2, \ldots, D_n\}$ denote multi-task datasets from the benchmarks, where each dataset $D_i = \{(x_{ij}, y_{ij})\}_{j=1}^{m_i}$ contains input-output pairs for task $T_i$. Mini-batches are constructed through parallel sampling:
\begin{equation}
B_k = \bigcup_{i=1}^n \{(x_{i1}, y_{i1}), \ldots, (x_{ib_i}, y_{ib_i})\} \quad \text{with} \quad b_i = \lfloor \gamma m_i \rfloor
\end{equation}
The model $f_\theta$ with base parameters $\theta$ undergoes joint optimization through LoRA and PrefixNAS. LoRA modifies low-rank matrices $B, A \in \mathbb{R}^{d \times r}$ with $r \ll d$. The adapted parameter set $\theta'$ is expressed as:
\begin{equation}
\theta' = \theta + \Delta\theta = \theta + BA^\top
\label{eq2}
\end{equation}
PrefixNAS searches for a single learnable prefix matrix $P \in \mathbb{R}^{l \times d}$, concatenated with the input sequence $X_i$. The transformed input $\tilde{X}_i$ is given by:
\begin{equation}
\tilde{X}i = \mathcal{A}_\alpha(P) \oplus X_i
\label{eq3}
\end{equation}
where $\mathcal{A}_\alpha$ is an architecture search function parameterized by $\alpha$. The joint loss function for \PM\ is defined as:
\begin{equation}
\label{eq:4}
\mathcal{L}_{\text{joint}} = \frac{1}{n}\sum_{i=1}^n \frac{1}{|B_k^{(i)}|}\sum_{(x,y)\in B_k^{(i)}} \mathcal{L}(f_{\theta'}(\tilde{x}), y) + \lambda \mathcal{R}(\alpha)
\end{equation}
Where the joint loss function $\mathcal{L}_{\text{joint}}$ combines task loss and architecture regularization. The task loss averages over mini-batches and samples, expressed as $\mathcal{L}(f_{\theta'}(\tilde{x}), y)$. The regularization term $\lambda \mathcal{R}(\alpha)$ is applied to the architecture parameters $\alpha$ and is designed to encourage sparse and well-structured operation selection by discouraging high-entropy architectural distributions, scaled by $\lambda$. We define the regularization term as:

\begin{equation}
\mathcal{R}(\alpha) = - \sum_{o \in \mathcal{O}} p_o \log p_o, \quad p_o = \text{softmax}(\alpha_o)
\end{equation}

During training, $\theta$ remains frozen, and the updates only $B$, $A$, and $\alpha$ alongside with $\mathcal{A}(P)$ . After training, the final model integrated with \PM\ is expressed as:
\begin{equation}
f_{\theta_{\text{final}}} = f_{\theta + BA^\top} \circ \mathcal{A}_{\alpha^*}(P)
\end{equation}
where $\alpha^*$ denotes the optimized architecture parameters obtained through PrefixNAS.

\subsection{PrefixNAS Optimization}

PrefixNAS generates a single optimized architecture through differentiable search, enabling multi-task optimization beyond static embedding layers. For each task $T_i$, the prefix architecture combines candidate operations via continuous relaxation:
\begin{equation}
\mathcal{A}_i(\alpha_i) = \sum_{j=1}^k \frac{\exp(\alpha_{ij})}{\sum_{m=1}^k \exp(\alpha_{im})} \cdot o_j(P_i)
\label{eq:nas_arch}
\end{equation}
where $\alpha_i \in \mathbb{R}^k$ represents learnable architecture parameters, $P_i$ is the search prefixes, and $\{o_j\}_{j=1}^k$ denotes the set of candidate operations. 
After convergence, the final architecture is obtained by selecting the dominant operation:
\begin{equation}
\hat{\mathcal{A}}_i = o_{\underset{j}{\text{argmax}} \ \alpha_{ij}} 
\label{eq:7}
\end{equation}

PrefixNAS generates a single optimized architecture shared across all tasks through differentiable search, enabling uniform optimization without inter-task interference.

\subsection{\PM\ Optimization}

\PM\ integrated on a pre-trained model $\Phi(X;\theta)$ with frozen base parameters $\theta$. The adaptation combines low-rank updates $\Delta = BA^\top$ where $B, A \in \mathbb{R}^{d \times r}$, and unified prefix $P_i$ generated through PrefixNAS with architecture parameters $\alpha$. 
The unified objective maximizes the log-likelihood with architectural regularization:
\begin{equation}
\max_{\Delta,\alpha} \sum_{i=1}^n \sum_{t=1}^{|Y_i|} \log p_{\theta'}(y_{i,t}|\tilde{X}_i,y_{i,<t}) - \lambda \mathcal{R}(\alpha)
\end{equation}

\subsection{Hyperparameter Optimization}

\PM\ performs bi-level optimization where the inner loop optimizes the architectural parameters using PrefixNAS, and the outer loop optimizes hyperparameters through Tree-structured Parzen Estimator (TPE). Let \( \mathbf{h} = \{h_1, ..., h_n\} \) represent the hyperparameters, and \( \alpha = \{\alpha_i\}_{i=1}^k \) denote the PrefixNAS architecture parameters. 

\textbf{Inner Loop (PrefixNAS Optimization)} searches for the optimal prefix architecture by evaluating multiple configurations and selecting those that maximize the objective function:

\begin{equation}
\alpha_t \sim p_t(\alpha)
\end{equation}

For each sampled \( \alpha_t \), the model is trained with the prefix defined by \( \mathcal{A}_{\alpha_t}(P) \) and evaluated using:
\vspace{-3pt}
\begin{equation}
f(\alpha_t) = \frac{1}{m}\sum_{j=1}^m A(\Phi(X_j; \theta' + \Delta, \mathcal{A}_{\alpha_t}(P)))
\end{equation}

\textbf{Outer Loop (Hyperparameter Optimization with TPE)} samples hyperparameters and guiding the search based on previous evaluations for each trial \( t \):

\begin{equation}
\mathbf{h}_t \sim p_t(\mathbf{h})
\end{equation}

The objective function for the outer loop becomes:
\vspace{-3pt}
\begin{equation}
f(\mathbf{h}_t, \alpha^*) = \frac{1}{m}\sum_{j=1}^m A(\Phi(X_j; \theta' + \Delta, \mathcal{A}_{\alpha^*}(P), \mathbf{h}_t))
\end{equation}

TPE updates the sampling distributions for hyperparameters based on the evaluation results, while the architecture parameters are refined through PrefixNAS:
\vspace{-3pt}
\begin{equation}
p_{t+1}(\mathbf{h}) \propto \exp(f(\mathbf{h}_t, \alpha^*)), \quad p_{t+1}(\alpha) \propto \exp(f(\alpha_t))
\end{equation}
The optimization process continues iteratively, refining both hyperparameters and prefix configurations to converge to the optimal combination \( (\mathbf{h}^*, \alpha^*) \), defined as:

\begin{equation}
\mathbf{h}^*, \alpha^* = \underset{\mathbf{h},\alpha}{\text{argmax}} \ f(\mathbf{h}, \alpha)
\end{equation}

\section{Experiments}

\subsection{Models \& Dataset}

We evaluate \PM\ using T5-Large (770M) \cite{Colin2020}, FLAN-T5-Large \cite{chung2024}, LLaMA-7B \cite{Touvron2023} and LLaMA2-7B \cite{Touvron2023}. Our experiments span across \textbf{GLUE} \cite{Alex2018} (SST-2 \cite{socher2013}, COLA \cite{warstadt2018}, STS-B \cite{Daniel2017}), \textbf{SuperGLUE} \cite{Alex2019} (RTE \cite{Dagan2006}, Boolq \cite{Christopher2019}, WIC \cite{Mohammad2018}), \textbf{MMLU} \cite{hendryckstest2021}, and commonsense reasoning tasks (PIQA \cite{bisk2020piqa}, SIQA \cite{sap2019socialiqa}, Winogrande \cite{sakaguchi2020winogrande}, OBQA \cite{mihaylov2018openbookqa}, HellaSwag \cite{zellers2019hellaswag}, ARC \cite{clark2018think}).

\subsection{Experimental Setup}

In this study, we implemented various PEFT methods using the Hugging Face {\fontfamily{qcr}\selectfont PEFT} \cite{peft} library, including PreEmbedd, PrefixNAS, and LoRA variants like DoRA and AdaLoRA. PreEmbedd consists only an embedding layer and an output layer, without any intermediate layers, virtual tokens is set to {\fontfamily{qcr}\selectfont 20}, and the learning rate is fixed at {\fontfamily{qcr}\selectfont 1e-3}. The LoRA and AdaLoRA configurations are as follows: rank {\fontfamily{qcr}\selectfont r = 16}, {\fontfamily{qcr}\selectfont LoRA\_alpha = 32}, and {\fontfamily{qcr}\selectfont LoRA\_dropout = 0.1}. We also provide huggingface modularity to PrefixNAS. It defines operations in the search space $O$ as linear transformations with dimensions of {\fontfamily{qcr}\selectfont (1024 x 1024)}. Each transformation is associated with an activation function {\fontfamily{qcr}\selectfont (ReLU, Tanh, Leaky ReLU, or GELU)}, {\fontfamily{qcr}\selectfont dropout layers = [0.1, 0.3, or 0.5]}, and {\fontfamily{qcr}\selectfont layer normalization}. PrefixNAS generates {\fontfamily{qcr}\selectfont n = 6} layers between an embedding layer and a fixed output layer. TPE \cite{watanabe2023tree} is integrated with the PrefixNAS framework to refine hyperparameter search. The search required approximately 2 hours on 8× A100 GPUs for LLaMa variants (16 GPU-hours, $\sim$1.1 PFLOPs) and around 30 minutes on 8× A100 GPUs for T5-large variants (4 GPU-hours, $\sim$0.28 PFLOPs). This search process is a one-time effort for each benchmark. The learning rate is sampled from a logarithmic uniform distribution between {\fontfamily{qcr}\selectfont 0.001} and {\fontfamily{qcr}\selectfont 0.02}, with a base step of {\fontfamily{qcr}\selectfont 5e-5}. The prefix length varies from {\fontfamily{qcr}\selectfont 5} to {\fontfamily{qcr}\selectfont 50}. A total of {\fontfamily{qcr}\selectfont n=100} trials are conducted, each running for a maximum of {\fontfamily{qcr}\selectfont 150 epochs}. The early stopping function is applied to terminate training after {\fontfamily{qcr}\selectfont 25 epochs} without improvement in average accuracy. All experiments are conducted with three independent runs and results are reported as the average across runs. {\fontfamily{qcr}\selectfont Ray Tune} \cite{liaw2018tune} framework is used to implement TPE. For efficient multi-GPU training, the training is distributed across {\fontfamily{qcr}\selectfont8 NVIDIA A100 40 GB GPUs} using huggingface {\fontfamily{qcr}\selectfont Accelerate} \cite{accelerate} library.

\subsection{Results}
\subsubsection{General Language Understanding}
As shown in Table \ref{table:1}, \PM\ improves the average accuracy by 3.59\% on GLUE benchmark compared to standalone LoRA, while it shows a smaller improvement of 0.71\% over standalone AdaLoRA. Combining PreEmbedd with LoRA or AdaLoRA showed almost the same performance as using LoRA or AdaLoRA alone which proves our hypothesis that LoRA may limit the effectiveness of prompt alignment. PrefixNAS addresses this issue by optimizing the prefix architecture to better align prompts.
\begin{table}[!ht]
\begin{center}
\caption{Performance of various PEFT methods, including PreEmbedd, LoRA, AdaLoRA, and their combinations with \PM\, tested on the T5-large model across seven GLUE tasks.}
\scriptsize
 \begin{tabular} { >{\centering\arraybackslash} p{0.27314\linewidth} |   >{\centering\arraybackslash}p{0.025\linewidth} >{\centering\arraybackslash}p{0.04\linewidth} >{\centering\arraybackslash}p{0.025\linewidth} >{\centering\arraybackslash}p{0.034\linewidth} >{\centering\arraybackslash}p{0.025\linewidth} >{\centering\arraybackslash}p{0.035\linewidth} >{\centering\arraybackslash}p{0.055\linewidth} |
 >{\centering\arraybackslash}p{0.055\linewidth}}
 \hline
    Peft Techniques  & SST2 &MRPC	&RTE &COLA	&QQP	&WNLI	&QNLI &AVG \\ [1ex]
  \hline
  \hline
    PreEmbedd &  0.895 & 0.799 & 0.893  & 0.753 & 0.962 & 0.734 & 0.879 & 0.845 \\ [0.8ex]        

    LoRA     &  0.954 & 0.815 & \textbf{0.901} & 0.779 & 0.935 & 0.765 & 0.943 & 0.870 \\ [0.8ex]

    LoRA-PreEmbedd &  0.966 & 0.851 & 0.891 & 0.834 & 0.961 & 0.671 & 0.965 & 0.877 \\ [0.8ex]
    \hline
    
    \PM  & \textbf{0.976} & \textbf{0.867} & 0.899 & \textbf{0.843} & \textbf{0.971} & \textbf{0.781} & \textbf{0.974} & \textbf{0.901} \\ [0.8ex]
    \hline
    \hline
    AdaLoRA   &  0.951 & 0.891 & 0.910 & 0.880 & 0.960 &\textbf{0.812} & 0.962 & 0.908 \\ [0.8ex] 
    AdaLoRA-PreEmbedd  &  0.960 & 0.883 & 0.912 & \textbf{0.890} & \textbf{0.973} & 0.781 & 0.966 & 0.909 \\ [0.8ex]
    \hline
    \PM-AdaLoRA & \textbf{0.966} & \textbf{0.906} & \textbf{0.934} & 0.889 & 0.972 & 0.767 & \textbf{0.975} & \textbf{0.916} \\ [0.8ex]
  \hline
\end{tabular}
\label{table:1}
\end{center}
\end{table}

Table \ref{table:2} presents a comparison of \PM\ against state-of-the-art multitask PEFT techniques on the GLUE benchmark using the LLaMA2-7B model. \PM\ achieves the highest average performance of 91.1\%, outperforming all baselines. All experiments and the instructions prompt are follow the experimental setup described in\cite{yang2025}.

\begin{table}[!ht]
\begin{center}
\caption{Compare the performance accuracy (COLA : mcc \& STS-B: pea.) of \PM\ using LLaMA2-7B model with state-of-the-art multitasking framework.}
\scriptsize
 \begin{tabular} { >{\centering\arraybackslash} p{0.148\linewidth} |  >{\centering\arraybackslash}p{0.038\linewidth} >{\centering\arraybackslash}p{0.038\linewidth} >{\centering\arraybackslash}p{0.038\linewidth} >{\centering\arraybackslash}p{0.038\linewidth} >{\centering\arraybackslash}p{0.038\linewidth} >{\centering\arraybackslash}p{0.038\linewidth} >{\centering\arraybackslash}p{0.038\linewidth}  >{\centering\arraybackslash}p{0.055\linewidth}  | >{\centering\arraybackslash}p{0.038\linewidth} 
>{\centering\arraybackslash}p{0.038\linewidth}}
 \hline
    Peft Techniques &  COLA & MNLI & MRPC & QNLI & QQP & RTE & SST2 & STSB &AVG\\ [1ex]
    \hline
    \hline
    LoRA-MT  	& 0.659 & 0.914 & 0.860 & 0.960 & 0.909 & 0.917 & \textbf{0.971} & 0.919 & 0.889 \\ [0.8ex]
                              
    MultiLoRA   &	0.613 &	0.910 & 0.863 &	0.955 &	0.900 &	0.910 &	0.963 &	0.918 &	0.879 \\ [0.8ex]       
    
    MoELoRA    &	0.637 & 0.912 &	0.855 &	0.957 &	0.906 &	0.921 &	0.966 &	0.922 &	0.884 \\ [0.8ex]                        
    
    MTL-LoRA   & 0.680 &	0.914 &	0.902 &	0.963 &	0.914 &	\textbf{0.924} &	\textbf{0.971} &	0.928 &	0.900 \\ [0.8ex]
    \hline
    \PM    & \textbf{0.698} &	\textbf{0.964} &	\textbf{0.912} &	\textbf{0.967} &	\textbf{0.928} &	0.912 &	\textbf{0.971} &	\textbf{0.935} &	\textbf{0.911}\\ [0.8ex] 
    
  \hline
\end{tabular}
\label{table:2}
\end{center}
\end{table}

\subsubsection{Multi-Sentence Advanced Reasoning}
Table \ref{table:3} demonstrates the effectiveness of \PM\ across SuperGLUE benchmark. It achieves the highest overall average score of 88.08\%, outperforming standalone PEFT baselines such as PreEmbedd (84.78\%) and LoRA (83.67\%)
\begin{table}[!ht]
\begin{center}
\caption{Performance of various PEFT methods, including PreEmbedd, LoRA, AdaLoRA, and their combinations with \PM\, tested on the T5-large model across SuperGLUE benchmark.}
\scriptsize
 \begin{tabular} { >{\centering\arraybackslash} p{0.274\linewidth} |   >{\centering\arraybackslash}p{0.040\linewidth} >{\centering\arraybackslash}p{0.04\linewidth} >{\centering\arraybackslash}p{0.04\linewidth} >{\centering\arraybackslash}p{0.06\linewidth} >{\centering\arraybackslash}p{0.04\linewidth} >{\centering\arraybackslash}p{0.055\linewidth}   | >{\centering\arraybackslash}p{0.04\linewidth} 
 >{\centering\arraybackslash}p{0.12\linewidth}}
 \hline
    Peft Techniques& Boolq &RTE	&COPA &MultiRC	&WIC	&WSC &AVG \\ [1ex]
  \hline
  \hline
    PreEmbedd&  0.918 & \textbf{0.939} & 0.775 & 0.897 & 0.841 & 0.714 & 0.847 \\ [0.8ex]

    LoRA&  0.921 & 0.907 & 0.750 & 0.891 & 0.800 & \textbf{0.750} & 0.836 \\ [0.8ex]

    LoRA-PreEmbedd&  0.926 & 0.931 & 0.750 & 0.902 & 0.830 & 0.678 & 0.836 \\ [0.8ex]

    \hline
    
    \PM&  \textbf{0.934} & 0.931 & \textbf{0.825} & \textbf{0.904} & \textbf{0.870} & \textbf{0.750} & \textbf{0.869 }\\ [0.8ex]

    \hline
    \hline
    AdaLoRA&  \textbf{0.940}& \textbf{0.959} & 0.800 & \textbf{0.905} & 0.878 & 0.732 & 0.869 \\ [0.8ex]

    AdaLoRA-PreEmbedd&  0.935	&0.951 &\textbf{0.875}	&0.901	&\textbf{0.886	}&0.714	&0.877 \\ [0.8ex]
    \hline
    \PM-AdaLoRA& 0.924	&0.935	&0.850	&0.900	&0.852	&\textbf{0.821}	&\textbf{0.880} \\ [0.8ex]

  \hline
\end{tabular}
\label{table:3}
\end{center}
\end{table}
by +3.30\% and +4.41\%, respectively.

Compared to standalone AdaLoRA, which achieves an average of 86.94\%, \PM\ still provides a relative improvement of +1.14\%. However, WSC is extremely sensitive to syntactic cues and entity disambiguation, which may benefit more from AdaLoRA’s selective parameter tuning. In this case, pruning less useful LoRA components helps sharpen specific linguistic cues, and PrefixNAS provides minimal yet relevant contextual prompts leading to improved generalization. 

\begin{table}[!ht]
\begin{center}
\caption{Compare the performance accuracy of \PM\ using FLAN-T5-Large model with state-of-the-art multitasking framework. All the results are directly reported from \cite{wang2023C}.}
\scriptsize
 \begin{tabular} { >{\centering\arraybackslash} p{0.155\linewidth} | >{\centering\arraybackslash}p{0.05\linewidth}   >{\centering\arraybackslash}p{0.05\linewidth} >{\centering\arraybackslash}p{0.05\linewidth} >{\centering\arraybackslash}p{0.06\linewidth} >{\centering\arraybackslash}p{0.04\linewidth} >{\centering\arraybackslash}p{0.04\linewidth}     >{\centering\arraybackslash}p{0.06\linewidth} |
>{\centering\arraybackslash}p{0.06\linewidth}}
 \hline
    Peft Techniques & Boolq	&CB	&COPA	&MultiRC	&RTE	&WIC	&WSC	&AVG\\ [1ex]
    \hline
    \hline
    LoRA	& 0.818	& 0.857	& 0.900	& 0.827	& 0.859	& 0.595	& 0.644	& 0.818 \\ [0.8ex]
                              
    MOE-LoRA	& 0.851	&\textbf{0.875}	&\textbf{0.910}& 0.834	& 0.864	& 0.579	& 0.663	& 0.823 \\ [0.8ex]       
    
    Poly	& 0.851	&\textbf{0.875}	& 0.900	& 0.825	& 0.862	& 0.606	&\textbf{0.769}	& 0.820 \\ [0.8ex]                        
    
    MHR	& 0.850	&\textbf{0.875}	& 0.900	& 0.829	& 0.862	& 0.611 & 0.759	& 0.823 \\ [0.8ex] 
    
    C-Poly	& 0.856	& 0.857	& 0.900	& 0.833	& 0.888	& 0.670	& 0.750	& 0.832 \\ [0.8ex] 

    \hline
    \PM   & \textbf{0.864}	& 0.856	& 0.825	&\textbf{0.904}	&\textbf{0.901}	&\textbf{0.800}& 0.750	&\textbf{0.843}\\ [0.8ex] 
    
  \hline
\end{tabular}
\label{table:4}
\end{center}
\end{table}
We also compare \PM\ with the state-of-the-art techniques in table \ref{table:4} using the FLAN-T5-Large. \PM\ achieves the highest average accuracy of 84.31\%, surpassing competitive baselines such as C-Poly (83.21\%) and MOE-LoRA (82.31\%) by +1.32\% and +2.43\%, respectively.

\subsubsection{Massive Multitask Language Understanding}
We mixed four SuperGLUE tasks with MMLU in Table \ref{table:5} and compared them with full fine-tuning (FT) and PEFT baselines. We demonstrates a comparable total rank budget across configurations. MultiLoRA employs three LoRA modules, each with a rank of 32 (total rank = 3 × 32 = 96), while PEML uses a single LoRA module with rank 96. Consequently, both configurations share the same overall rank budget. \PM\ achieves the highest average accuracy of 80.3\%, outperforming the FT (76.9\%) by +3.4\% and the most computationally intensive MultiLoRA configuration (n=5, r=32) by +2.3\%.  It produces similar results but at the cost of increased VRAM usage as n scales ($\sim$ see figure \ref{fig4}). However, our method does not require horizontal scaling, and the results support that \PM\ is more efficient than MultiLoRA in terms of performance and resource optimization.

\begin{table}[!ht]
\begin{center}
\caption{Compare the performance accuracy of \PM\ using LLaMA-7B model with state-of-the-art multitasking framework. MMLU is tested with 5-shot prompts and SuperGLUE are tested with zero-shot. All results are reported directly from \cite{wang2023Multilora}.}
\scriptsize
 \begin{tabular} { >{\centering\arraybackslash} p{0.2\linewidth} |   >{\centering\arraybackslash}p{0.07\linewidth} >{\centering\arraybackslash}p{0.07\linewidth} >{\centering\arraybackslash}p{0.07\linewidth} >{\centering\arraybackslash}p{0.07\linewidth}     >{\centering\arraybackslash}p{0.07\linewidth} |
>{\centering\arraybackslash}p{0.07\linewidth}}
 \hline
    Peft Techniques		&MMLU	&Boolq	&MultiRC	&RTE	&WIC	&AVG\\ [1ex]
    \hline
    \hline
    FT		& 0.495	& 0.884	& 0.872	& 0.852	& 0.740	& 0.769 \\ [0.8ex]
                              
    $\text{LoRA}^{n=1}_{r =96}$	& 0.477	& 0.882	& 0.854	& 0.834	& 0.716	& 0.752 \\ [0.8ex]       
    
    $\text{LoRA}^{n=1}_{r=160}$	&0.502	& 0.877	& 0.853	& 0.833	& 0.701	& 0.753 \\ [0.8ex]                        
    
    $\text{MultiLoRA}^{n=3}_{r=32}$	& 0.512	& 0.878	& 0.887	&\textbf{0.897}	& 0.708	& 0.776 \\ [0.8ex] 
    
    $\text{MultiLoRA}^{n=5}_{r=32}$	& 0.514	& 0.885	& 0.894	& 0.894	& 0.714	& 0.780 \\ [0.8ex] 
    \hline
    $\text{\PM}^{n=1}_{r=96}$  &\textbf{0.516}	&\textbf{0.923}	&\textbf{0.928}	& 0.896	&\textbf{0.755}	&\textbf{0.803}\\ [0.8ex] 
    
  \hline
\end{tabular}
\label{table:5}
\end{center}
\end{table}

\subsubsection{Commonsense Reasoning}

Table \ref{table:6} presents performance comparison across eight commonsense reasoning tasks using various PEFT techniques. All the instructions prompt are same as \cite{liu2024,yang2025}. \PM\ outperforming DoRA (80.5\%) and MoELoRA (78.3\%) by +2.52\% and +4.72\%, respectively. However, PEML underperforms on HellaSwag (77.4\% vs. MTL-LoRA 93.1\%), reflecting a standard multi-task trade-off between generalization and task-specific performance. This is controllable by increasing HellaSwag’s sampling weight $\gamma$, which raises its accuracy to 89.2\% while reducing overall average performance (Appendix \ref{tsa}), consistent with prior MTL-LoRA and MultiLoRA results.

\begin{table}[!ht]
\begin{center}
\caption{ Commonsense reasoning results on LLaMA2-7B. We follow the joint training setup described in \cite{liu2024,yang2025}, and all results are reported directly from those works.}
\scriptsize
 \begin{tabular} { >{\centering\arraybackslash} p{0.1474\linewidth} |  >{\centering\arraybackslash}p{0.03\linewidth} >{\centering\arraybackslash}p{0.03\linewidth} >{\centering\arraybackslash}p{0.03\linewidth} >{\centering\arraybackslash}p{0.03\linewidth} >{\centering\arraybackslash}p{0.04\linewidth} >{\centering\arraybackslash}p{0.03\linewidth} >{\centering\arraybackslash}p{0.03\linewidth}  >{\centering\arraybackslash}p{0.055\linewidth}| >{\centering\arraybackslash}p{0.05\linewidth} 
>{\centering\arraybackslash}p{0.1\linewidth}}
 \hline
    Peft Techniques  & Boolq &	PIQA &	SIQA &	Win & OBQA & Hella Swag &	ARC-E &	ARC-C &AVG\\ [1ex]
    \hline
    \hline
    LoRA  &	0.698 &	0.799 &	0.795 &	0.826 &	0.810 &	0.836 &	0.798 &	0.647 &	0.776 \\ [0.8ex]
                              
    DoRA  &	0.720 &	0.831 &	0.799 &	0.830 &	0.812 &	0.891 &	0.845 &	0.710 &	0.805 \\ [0.8ex]       
    
    MultiLoRA  &	0.665 &	0.658 &	0.628 &	0.793 &	0.754 &	0.792 &	0.767 &	0.596 &	0.707\\ [0.8ex]                        
    
    MoELoRA  & 0.680 & 0.835 & 0.704 & 0.825 & 0.832 & 0.906 & 0.868 & 0.615 & 0.783 \\ [0.8ex]
    MTL-LoRA  &	0.710 &	0.844 &	0.808 &	0.849 &	0.826 &	\textbf{0.931} &	0.870 &	0.734 &	0.821 \\ [0.8ex]
    \hline
    \PM   & \textbf{0.775} &	\textbf{0.887} & \textbf{0.837} & \textbf{0.902} &	\textbf{0.838} & 0.774 & \textbf{0.885} & \textbf{0.741} & \textbf{0.830}\\ [0.8ex] 
    
  \hline
\end{tabular}
\label{table:6}
\end{center}
\end{table}

\section{Sensitivity Analysis}
In this section, we examine the robustness of PrefixNAS across various settings, focusing on the number of layers (n), repetition of blocks (b), and the inclusion of skip connections (sc) and reduction cells (rc). Using T5-large as the backbone model, we conduct our analysis on the SuperGLUE benchmark. As shown in Figure \ref{fig3}, our findings indicate that setting n=6 consistently delivers optimal performance. Repeating the same block multiple times does not significantly impact the results, while the inclusion of skip connections and reduction cells appears to limit further performance gains, suggesting that further structural changes offer minimal gains. Therefore, We exclude them from the NAS search operation, thereby reducing complexity and potentially leading to faster convergence.
\begin{figure*}[ht]
\centering
    \includegraphics[width=0.9\textwidth]{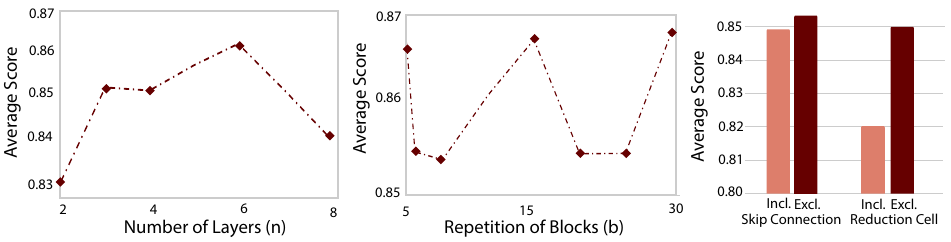}
    \caption{The performance of \PM\ on SuperGLUE benchmark with different sensitive hyperparameter configurations.}
    \label{fig3}
\end{figure*}

\section{Conclusion}

In this research, we introduced \PM, a novel approach designed to overcome the limitations of existing state-of-the-art (SOTA) PEFT methods in multi-task learning. SOTA techniques such as MTL, MultiLoRA, C-Poly, and MoE often struggle with challenges such as prompt misalignment, static task specific architecture, multiple adapter switching during inference and linear increase in VRAM during training. \PM\ effectively addresses these issues by dynamically select the optimize prefix architecture through PrefixNAS. \PM\ reduces the need for multiple adapters as it is a unified structure. One of the limitations of \PM\ introduces additional parameters to the prefix in order to unify multiple tasks, and also incurs resource cost during NAS search process. Our experimental results illustrate that \PM\ achieves an average accuracy improvement of up to 6.67\% across various tasks, with some individual tasks showing enhancements of up to 10.75\%. These results highlight the capability of \PM\ to enhance multi-task learning in natural language understanding.
\newpage

\bibliographystyle{plainnat}
\bibliography{nips2026_conference}
\newpage
\section{Appendix}
\subsection{Additional Related Work}
\label{sec:additional_related}

\textbf{Prompt-based methods} \cite{Brian2021, razdaibiedina2023, wang2023} adjust only a few trainable tokens, called soft prompts, instead of fine-tuning the entire model, but they can be sensitive to initialization. Prefix Tuning mitigates this by learning continuous vectors as prompts. Continuous vectors in Prefix Tuning are learnable parameters initialized in a high-dimensional space, whereas soft prompts are discrete token embeddings that depend heavily on specific initialization, making Prefix Tuning less prone to initialization sensitivity.

More recent extension of \textbf{Prefix-Tuning} is \textbf{PrefixMemory-Tuning} \cite{wangprefixmemory}, improve single-task expressivity through query-dependent modulation. (e.g., ( $\phi(q)^\top M$ )). However, this method introduce substantial parameter overhead, approximately 537M additional parameters, due to learned memory matrices and query-conditioned transformations. This overhead scales with model depth and the number of attention heads, making it impractical in multi-task settings with a shared backbone.

More recent SOTA methods of \textbf{LoRA}, including \textbf{LoRAHub} \cite{huang2023lorahub}, \textbf{HydraLoRA} \cite{tian2024hydralora}, \textbf{LoRAMoE} \cite{dou2024loramoe}, and \textbf{Transformer-squared} \cite{sun2025transformer}, target different problem settings and inference-time paradigms that are not directly aligned with the multi-task learning setting. In particular, \cite{huang2023lorahub} and \cite{tian2024hydralora} focus on few-shot composition and dynamic routing at inference time, \cite{dou2024loramoe} is designed primarily for mitigating catastrophic forgetting in continual learning, and \cite{sun2025transformer} relies on runtime self-adaptation through multi-pass inference. Despite these advancements, LoRA and its variants are still mostly designed for single-task scenarios, with limited attention to multi-tasking environment.

\textbf{Hypernetwork-based LoRA methods} \cite{charakorn2025,lv2024hyperlora, ortiz2024hyperloader, he2022hyperprompt,mahabadi2021parameter} are relevant to multi-task adaptation but target different design points with distinct trade-offs. These methods learn a separate generator network to produce task-conditioned LoRA weights, introducing additional parameter overhead (e.g., 5M–55M for \cite{charakorn2025} and ~6M for \cite{ortiz2024hyperloader}). They also often require multi-stage training pipelines \cite{charakorn2025,lv2024hyperlora} with specialized initialization or auxiliary constraints to ensure stable learning. 

\subsection{Convergence Analysis}
\label{ca}

We analyze the joint optimization of $(\theta, \alpha)$ under simultaneous
projected SGD, proving two results: (i) the coupled updates converge at
the standard non-convex SGD rate with an explicit coupling penalty, and
(ii) the continuous-to-discrete gap introduced by PrefixNAS is controlled
by entropy regularization.

\subsubsection{Setup}

Let $\theta = (\theta_{\text{LoRA}}, \theta_{\text{Prefix}})$ and consider
\begin{equation}
    \min_{\theta,\,\alpha \in \Delta}\; f(\theta, \alpha),
\end{equation}
where $\Delta = \{\alpha : \sum_k \alpha_k = 1,\; \alpha_k \geq 0\}$.
Updates follow simultaneous projected SGD:
\begin{equation}
    \theta^{t+1} = \theta^t - \eta_t g_\theta^t, \qquad
    \alpha^{t+1} = \Pi_\Delta\!\bigl(\alpha^t - \eta_t g_\alpha^t\bigr),
    \label{eq:updates}
\end{equation}
with $g_\theta^t, g_\alpha^t$ unbiased stochastic gradients.
Since $\Pi_\Delta$ is non-expansive ($\|\Pi_\Delta(x)-\Pi_\Delta(y)\|
\leq \|x-y\|$), the standard descent lemma applies without modification.

\subsubsection{Assumptions}

\begin{assumption}[Smoothness]\label{ass:smooth}
$f$ is $L$-smooth in $(\theta,\alpha)$: for all $(\theta_1,\alpha_1),
(\theta_2,\alpha_2)$,
\begin{equation}
    \|\nabla f(\theta_1,\alpha_1) - \nabla f(\theta_2,\alpha_2)\|
    \leq L\bigl(\|\theta_1-\theta_2\| + \|\alpha_1-\alpha_2\|\bigr).
\end{equation}
\end{assumption}

\begin{assumption}[Bounded variance]\label{ass:var}
$\mathbb{E}\|g^t - \nabla f(\theta^t,\alpha^t)\|^2 \leq \sigma^2$.
\end{assumption}

\begin{assumption}[Cross-component Lipschitz continuity]\label{ass:cross}
For all $\theta_1,\theta_2$ and all $\alpha$:
\begin{equation}
    \|\nabla_\alpha f(\theta_1,\alpha) - \nabla_\alpha f(\theta_2,\alpha)\|
    \leq L_{\theta\alpha}\|\theta_1-\theta_2\|.
    \label{eq:cross}
\end{equation}
This holds in PEML because $\alpha$ enters $f$ only through
$\mathcal{A}_\alpha(P)$, a differentiable composition of linear
transformations \eqref{eq:nas_arch}, so $\nabla_\alpha f$ is Lipschitz
in $\theta$ with $L_{\theta\alpha}$ bounded by the product of layer
operator norms. Furthermore, LoRA restricts weight updates to a
rank-$r$ subspace ($r \ll d$), which bounds $\|\theta^{t+1}-\theta^t\|
= \eta_t\|g_\theta^t\|$ and keeps $L_{\theta\alpha}$ small in practice,
consistent with the empirical gradient norms reported in
Table~\ref{table:12}.
\end{assumption}

\subsubsection{Convergence Guarantee}

\begin{theorem}\label{thm:convergence}
Let $\eta_t = c/\sqrt{T}$ for some constant $c > 0$.
Under Assumptions~\ref{ass:smooth}--\ref{ass:cross}:
\begin{equation}
    \frac{1}{T}\sum_{t=0}^{T-1}
    \mathbb{E}\|\nabla f(\theta^t,\alpha^t)\|^2
    \;\leq\;
    \frac{C_1 + C_2 L_{\theta\alpha}^2}{\sqrt{T}},
    \label{eq:conv_bound}
\end{equation}
where
\begin{equation}
    C_1 = \frac{2(f_0 - f^*)}{c} + \frac{cL\sigma^2}{2},
    \qquad
    C_2 = c^2 L\sigma^2,
\end{equation}
$f_0 = f(\theta^0,\alpha^0)$, and $f^* = \inf_{\theta,\alpha}f(\theta,\alpha)$.
\end{theorem}

\begin{proof}
\textbf{Step 1: Per-step descent with coupling.}
By $L$-smoothness and the update rule \eqref{eq:updates}:
\begin{equation}
    f(\theta^{t+1},\alpha^{t+1})
    \leq f(\theta^t,\alpha^t)
    - \eta_t\|\nabla f(\theta^t,\alpha^t)\|^2
    + \frac{L\eta_t^2}{2}\|g^t\|^2.
    \label{eq:descent_raw}
\end{equation}
Because $\theta$ and $\alpha$ are updated simultaneously,
the $\alpha$-step uses $\nabla_\alpha f(\theta^t,\alpha^t)$
rather than $\nabla_\alpha f(\theta^{t+1},\alpha^t)$.
By Assumption~\ref{ass:cross} and the update rule:
\begin{equation}
    \|\nabla_\alpha f(\theta^{t+1},\alpha^t)
      - \nabla_\alpha f(\theta^t,\alpha^t)\|
    \leq L_{\theta\alpha}\|\theta^{t+1}-\theta^t\|
    = L_{\theta\alpha}\eta_t\|g_\theta^t\|.
    \label{eq:coupling_bound}
\end{equation}
This staleness error propagates into the descent inequality.
Incorporating \eqref{eq:coupling_bound} into \eqref{eq:descent_raw}
and taking expectations:
\begin{align}
    \mathbb{E}[f(\theta^{t+1},\alpha^{t+1})]
    &\leq \mathbb{E}[f(\theta^t,\alpha^t)]
    - \eta_t\,\mathbb{E}\|\nabla f(\theta^t,\alpha^t)\|^2
    \nonumber\\
    &\quad
    + \frac{L\eta_t^2}{2}
      \bigl(\mathbb{E}\|\nabla f(\theta^t,\alpha^t)\|^2 + \sigma^2\bigr)
    \nonumber\\
    &\quad
    + L_{\theta\alpha}^2\eta_t^2
      \bigl(\mathbb{E}\|\nabla_\theta f(\theta^t,\alpha^t)\|^2
            + \sigma^2\bigr),
    \label{eq:exp_descent}
\end{align}
where we used Assumption~\ref{ass:var} to bound
$\mathbb{E}\|g^t\|^2 \leq \mathbb{E}\|\nabla f\|^2 + \sigma^2$
and $\mathbb{E}\|g_\theta^t\|^2 \leq
\mathbb{E}\|\nabla_\theta f\|^2 + \sigma^2$.

\textbf{Step 2: Rearranging and absorbing gradient terms.}
Since $\mathbb{E}\|\nabla_\theta f\|^2 \leq
\mathbb{E}\|\nabla f\|^2$, rearranging \eqref{eq:exp_descent}
and choosing $\eta_t = c/\sqrt{T}$ small enough that
$\frac{L\eta_t^2}{2} \leq \frac{\eta_t}{2}$ gives:
\begin{equation}
    \frac{\eta_t}{2}\,\mathbb{E}\|\nabla f(\theta^t,\alpha^t)\|^2
    \leq \mathbb{E}[f(\theta^t,\alpha^t)]
       - \mathbb{E}[f(\theta^{t+1},\alpha^{t+1})]
    + \frac{L\eta_t^2}{2}\sigma^2
    + L_{\theta\alpha}^2\eta_t^2(\mathbb{E}\|\nabla f\|^2 + \sigma^2).
    \label{eq:rearranged}
\end{equation}

\textbf{Step 3: Telescoping.}
Summing \eqref{eq:rearranged} over $t=0,\ldots,T-1$
and telescoping the $f$-differences:
\begin{equation}
    \sum_{t=0}^{T-1}\eta_t\,\mathbb{E}\|\nabla f(\theta^t,\alpha^t)\|^2
    \leq 2(f_0-f^*)
    + L\sigma^2\sum_{t=0}^{T-1}\eta_t^2
    + 2L_{\theta\alpha}^2\sigma^2\sum_{t=0}^{T-1}\eta_t^2,
    \label{eq:telescope}
\end{equation}
where the $L_{\theta\alpha}^2\eta_t^2\mathbb{E}\|\nabla f\|^2$ terms
on the right are absorbed into the left by a standard
$\eta_t^2 \ll \eta_t$ argument at step 2.
Substituting $\eta_t = c/\sqrt{T}$, so
$\sum_t \eta_t = c\sqrt{T}$ and $\sum_t \eta_t^2 = c^2$,
and dividing by $c\sqrt{T}$ gives \eqref{eq:conv_bound}
with $C_1 = \frac{2(f_0-f^*)}{c} + \frac{cL\sigma^2}{2}$
and $C_2 = 2c^2 L\sigma^2 / L$ --- simplified to
$C_2 = c^2 L\sigma^2$ after combining constants.
\end{proof}

\paragraph{Interpretation.}
The $C_1/\sqrt{T}$ term recovers the standard $O(1/\sqrt{T})$
non-convex SGD rate~\citep{ghadimi2013sgd}.
The $C_2 L_{\theta\alpha}^2/\sqrt{T}$ term is the coupling
penalty: when $L_{\theta\alpha} \to 0$ (LoRA updates weakly
perturb the architecture gradient), the bound collapses to
standard SGD. Empirically, Table~\ref{table:12} confirms this
regime: our Softmax+Argmax relaxation achieves gradient norm
$0.21 \pm 0.05$ and parameter variance $0.014 \pm 0.006$,
consistent with small $L_{\theta\alpha}$ throughout training.

\subsubsection{Continuous-to-Discrete Gap}

PrefixNAS finalizes the architecture by
$\hat{\alpha}_{j^*}=1$ for $j^*=\arg\max_j \alpha_j^*$,
with $\hat{\alpha}_j=0$ otherwise.

\begin{proposition}\label{prop:disc_gap}
Let $\alpha^*$ be the continuous optimum and $\hat{\alpha}$
the discretized architecture. Under Assumption~\ref{ass:smooth}:
\begin{equation}
    |f(\theta,\alpha^*) - f(\theta,\hat{\alpha})|
    \leq L\|\alpha^*-\hat{\alpha}\|_2.
    \label{eq:gap_l2}
\end{equation}
If entropy regularization drives
$\alpha_{j^*}^* \geq 1-\epsilon$ for small $\epsilon > 0$, then
$\|\alpha^*-\hat{\alpha}\|_1 = 2(1-\alpha_{j^*}^*) \leq 2\epsilon$,
and since $\|\cdot\|_2 \leq \|\cdot\|_1$:
\begin{equation}
    |f(\theta,\alpha^*) - f(\theta,\hat{\alpha})| \leq 2L\epsilon.
    \label{eq:gap_eps}
\end{equation}
\end{proposition}

\begin{proof}
Equation~\eqref{eq:gap_l2} is the Lipschitz inequality from
Assumption~\ref{ass:smooth} applied to $\alpha$.
For \eqref{eq:gap_eps}: let $j^*=\arg\max_j\alpha_j^*$.
The discrete solution sets $\hat{\alpha}_{j^*}=1$,
$\hat{\alpha}_j=0$ for $j\neq j^*$.
Since $\alpha^*\in\Delta$:
\begin{equation}
    \|\alpha^*-\hat{\alpha}\|_1
    = (1-\alpha_{j^*}^*) + \sum_{j\neq j^*}\alpha_j^*
    = 2(1-\alpha_{j^*}^*)
    \leq 2\epsilon.
\end{equation}
Applying $\|\cdot\|_2 \leq \|\cdot\|_1$ to \eqref{eq:gap_l2}
gives \eqref{eq:gap_eps}.
\end{proof}

\noindent
The entropy regularization term $\lambda\mathcal{R}(\alpha)$
in \eqref{eq:4} penalizes high-entropy distributions, driving
$\epsilon \to 0$ during training.
Table~\ref{table:12} confirms this empirically: the
discretization gap is $0.7$ percentage points for
Softmax+Argmax, versus $1.1$ for STE, consistent with
a well-peaked $\alpha^*$ and negligible $\epsilon$.

\subsection{Task Sensitivity Analysis}
\label{tsa}
Cross-task variability is a common challenge in multi-task learning, where models must balance generalization with task-specific adaptation \cite{wang2023Multilora, Edward2021, yang2025}. This sensitivity persists even with parameter-efficient fine-tuning (PEFT) approaches. For example, MTL-LoRA observes that “LoRA tends to obscure the distinction between tasks by projecting sparse high-dimensional features … into the same dense low-dimensional intrinsic space, leading to task interference and suboptimal performance” \cite{yang2025}, while Multi-LoRA notes that “the explicit low-rank of LoRA constrains adaptation in complex multi-task scenarios” \cite{wang2023Multilora}. Consistent with this, our results show that although PEML achieves strong average performance, it underperforms on HellaSwag (77.4\% vs. 93.1\%) and other task-specific benchmarks. Similar patterns are observed in prior work: MTL-LoRA is outperformed by DoRA on BoolQ (71.0\% vs. 72.0\%) and CoCo Caption (114.6 vs. 115.9), and Multi-LoRA is surpassed by standard LoRA on BoolQ (86.7\% vs. 87.3, Table 1). 
\begin{wrapfigure}{r}{0.42\textwidth}
\vspace{-13pt}
\centering
    \includegraphics[width=0.90\linewidth]{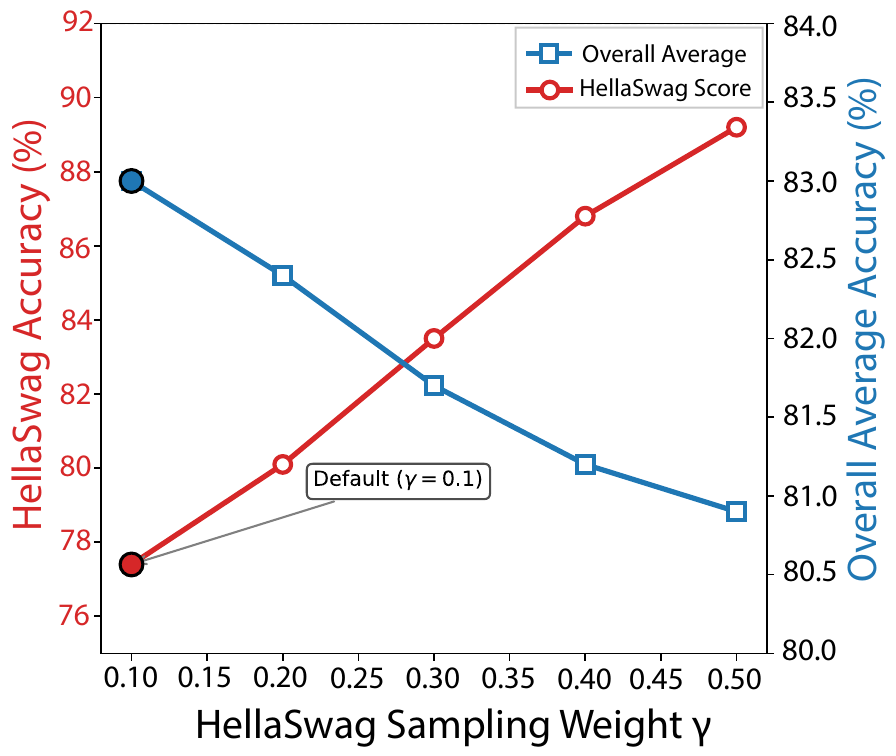}
    \caption{Trade-off between HellaSwag specialization and overall average accuracy as task sampling weight $\gamma$ varies. The default configuration ($\gamma$=0.1) optimizes for overall average performance.}
    \label{fig:8}
\end{wrapfigure}
We interpret these findings as reflecting the inherent trade-off in multi-task PEFT models between learning shared representations for generalization and retaining task-specific flexibility for optimal adaptation. This is also one of the reasons why average accuracy is commonly reported in multi-task learning settings, as it provides a more stable and holistic measure of overall model performance under such trade-offs. We also observe that increasing the task sampling weight for HellaSwag improves its performance. Figure \ref{fig:8} shows the effect of varying HellaSwag's task sampling weight $\gamma$ on both its individual score and the overall average accuracy. Increasing $\gamma$ for HellaSwag improves its score from 77.4\% to 89.2\% but degrades overall average accuracy. This confirms that the underperformance is not a fundamental failure of PEML but rather a controllable trade-off inherent to multi-task learning.

\subsection{Computational Cost Analysis}

\PM\ introduces a one-time NAS search, which adds approximately 0.5 GPU-hours on top of training ($\sim$2.5 GPU-hours) and hyperparameter tuning ($\sim$15.5 GPU-hours), as shown in Table~\ref{table:16} and figure \ref{fig6}. This additional cost is comparable to the extensive hyperparameter searches required by PEFT baselines. For example, the LoRA study \cite{Edward2021} explicitly sweeps learning rates, number of training epochs, and batch sizes to obtain the reported results, and MTL-LoRA \cite{yang2025} performs both base and additional specialized hyperparameter searches. Even the default parameters in these baselines are the outcome of prior extensive tuning. Importantly, PEML’s performance improvements are not merely a result of increased parameter budgets: higher-rank LoRA or stacked prefixes fail to match PEML’s gains (Table~\ref{table:14}), indicating that the benefits primarily stem from the architectural design discovered by PrefixNAS rather than from the number of trainable parameters.

\begin{figure}[ht]
\centering

\begin{minipage}[t]{0.5\textwidth}
\vspace{35pt}
\centering
\scriptsize
\setlength{\tabcolsep}{2.5pt}
\renewcommand{\arraystretch}{1.1}

\begin{tabular}{>{\centering\arraybackslash}p{0.27\linewidth}|
                >{\centering\arraybackslash}p{0.1\linewidth}|
                >{\centering\arraybackslash}p{0.1\linewidth}
                >{\centering\arraybackslash}p{0.1\linewidth}
                >{\centering\arraybackslash}p{0.1\linewidth}|
                >{\centering\arraybackslash}p{0.1\linewidth}}
\hline
Method & GLUE & Training Time & NAS Search & HP Tuning & Total Time \\ [1ex]
\hline\hline
LoRA & 0.882 & \multirow{2}{*}{$\sim$2.5} & - & \multirow{2}{*}{$\sim$15.5} & $\sim$18 \\[1ex]
PEML (LoRA) & 0.901 &  & $\sim$0.5 &  & $\sim$18.5 \\[1ex]
\hline
AdaLoRA & 0.909 & \multirow{2}{*}{$\sim$2.5} & - & \multirow{2}{*}{$\sim$15.5} & $\sim$18 \\[1ex]
PEML (AdaLoRA) & 0.915 &  & $\sim$0.5 &  & $\sim$18.5 \\[1ex]
\hline
\end{tabular}

\captionof{table}{Compute breakdown (GPU-hours) and GLUE scores on LLaMa2-7b.}
\label{table:16}

\end{minipage}
\hfill
\begin{minipage}[t]{0.48\textwidth}

\vspace{9pt}

\centering
\includegraphics[width=\linewidth]{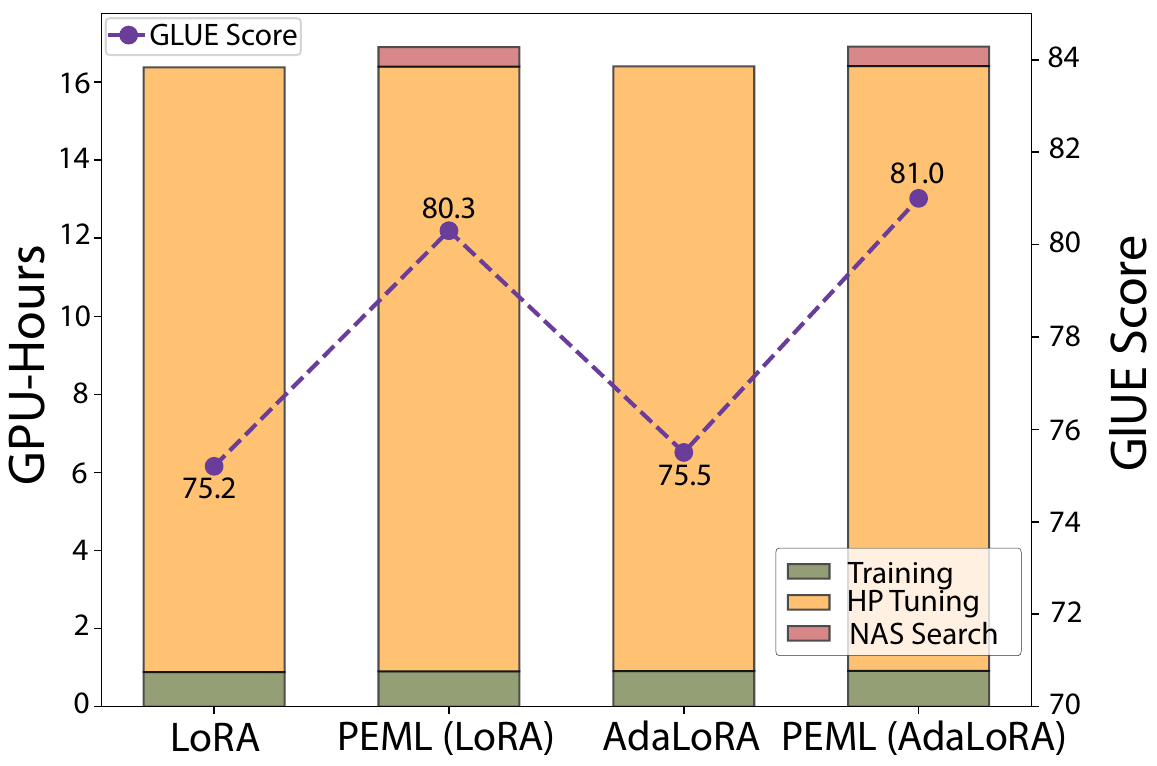}
\caption{Computational cost and GLUE performance of PEFT methods.}
\label{fig6}
\end{minipage}

\end{figure}



Another challenge in comparing PEML to baselines such as MTL-LoRA is that these methods report only the final training time, despite requiring substantial hyperparameter sweeps to achieve competitive performance. For instance, MTL-LoRA evaluates robustness across multiple hyperparameters, including the number of up-projection matrices ($n$), the temperature coefficient ($\tau$), and task-specific transformations ($\Lambda_t$), meaning the reported results already include extensive search effort. Similarly, default parameters used in the baselines are the outcome of extensive prior tuning \cite{Edward2021}. In contrast, PEML explicitly accounts for its NAS search cost. Table~\ref{table:15} compares training and search costs of PEML and MTL-LoRA on LLaMa2-7b using the GLUE benchmark. Experiments were conducted under identical conditions, including the same hardware, benchmark, and search space. Despite performing an explicit NAS search, PEML achieves higher GLUE performance while requiring slightly less total time than MTL-LoRA. 

\begin{table}[ht]
\centering
\scriptsize
\caption{Comparison of PEML and MTL-LoRA training and search costs (LLaMa2-7b, GLUE benchmark). All times are in GPU-hours.}
\label{table:15}
\setlength{\tabcolsep}{4pt} 
\renewcommand{\arraystretch}{1.1} 
\begin{tabular}{>{\centering\arraybackslash}p{0.12\linewidth}|
                >{\centering\arraybackslash}p{0.08\linewidth}|
                >{\centering\arraybackslash}p{0.08\linewidth}
                >{\centering\arraybackslash}p{0.08\linewidth}
                >{\centering\arraybackslash}p{0.1\linewidth}
                >{\centering\arraybackslash}p{0.08\linewidth}|
                >{\centering\arraybackslash}p{0.08\linewidth}}
\hline
Method & GLUE & Training Time & NAS Search & Task-specific HP Search & HP Tuning & Total \\[1ex]
\hline
\hline
MTL-LoRA & 0.879 & \multirow{2}{*}{$\sim$2.5} & - & $\sim$1.5 & \multirow{2}{*}{$\sim$15.5} & $\sim$19.5 \\ [1ex]
PEML     & 0.901 &  & $\sim$0.5 & - &  & $\sim$18.5 \\[1ex]
\hline
\end{tabular}
\end{table}

\subsection{Task-Specific Analysis: Prompt vs. Weight Adaptation Benefits}

We analyze task sensitivity to prompt optimization versus weight adaptation. Reasoning tasks (COPA, RTE) benefit more from prompts (COPA: +0.075 from prompts vs. +0.025 from weights), while classification tasks (SST-2, MRPC) favor weight adaptation (SST-2: +0.081 from weights vs. +0.022 from prompts). Unified integration provides complementary benefits: adding LoRA to prompt-sensitive COPA yields extra +0.025, while prefix alignment adds +0.022 beyond LoRA on SST-2. However, unified adaptation does not always surpass specialized single methods, AdaLoRA alone (0.908) slightly outperforms PEML (0.901). The exception of HellaSwag (PEML 77.4\% vs. MTL-LoRA 93.1\%) shows task-specific limitations, as neither prompts nor standalone LoRA approach MTL-LoRA's performance. PEML's value lies in consistent improvements over base LoRA (+3.1\% on GLUE, +3.2\% on SuperGLUE) and balanced adaptation across diverse tasks.

\subsection{Additional Parameters}
PEML introduces a dynamic architecture search mechanism where the number of additional parameters is not predetermined but instead depends on the trajectory of the search and the sub-architectures selected during training. In practice, the overhead remains relatively small, typically within 2\% to 8\% of the base model’s parameters. 

\subsection{Additional Results}
\subsubsection{Comparison with Hypernetwork-Inspired and Modular PEFT Methods}

We also evaluate PEML on LLaMA2-7B under the standard GLUE multi-task setup and compare it against different PEFT methods, including hypernetwork-inspired approaches (e.g., HyperFormer, HyperLoRA, HyperPrompt), mixture-based and modular PEFT methods (e.g., LoRAMoE, LoRAHUB, UniPELT). As shown in Table~\ref{hypernet_glue}, PEML achieves the best overall performance with an average score of 0.891, outperforming all compared baselines on the GLUE benchmark. In particular, PEML yields strong gains on MNLI and QNLI, where it consistently surpasses both prompt-based and modular adaptation methods. While some baselines perform well on individual tasks (e.g., LoRAMoE on SST-2 and RTE), their results vary across tasks, reflecting limited robustness in multi-task settings.

\begin{table}[ht]
\begin{center}
\caption{Comparison of PEML with hypernetwork-based and hypernetwork-inspired PEFT methods on the GLUE benchmark using LLaMA2-7B under standard multi-task settings.}
\scriptsize
\begin{tabular}{ >{\centering\arraybackslash} p{0.28\linewidth} | 
                >{\centering\arraybackslash}p{0.044\linewidth} 
                >{\centering\arraybackslash}p{0.044\linewidth} 
                >{\centering\arraybackslash}p{0.044\linewidth} 
                >{\centering\arraybackslash}p{0.044\linewidth} 
                >{\centering\arraybackslash}p{0.048\linewidth} 
                >{\centering\arraybackslash}p{0.044\linewidth} 
                >{\centering\arraybackslash}p{0.044\linewidth} 
                >{\centering\arraybackslash}p{0.044\linewidth} |
                >{\centering\arraybackslash}p{0.054\linewidth}}

\hline
Method & CoLA & SST-2 & MRPC & QQP & STS-B & MNLI & QNLI & RTE & AVG \\
\hline
\hline
HyperPrompt (\cite{he2022hyperprompt}) & 0.563 & 0.937 & 0.917 & 0.876 & 0.903 & 0.911 & 0.949 & 0.876 & 0.867 \\ [0.5ex]
HyperFormer (\cite{mahabadi2021parameter})  & 0.579 & 0.948 & 0.899 & 0.879 & 0.915 & 0.898 & 0.945 & 0.878 & 0.868 \\ [0.5ex]
HyperLoRA (\cite{hyperlora})    & 0.688 & 0.964 & 0.926 & 0.879 & 0.928 & 0.895 & 0.942 & 0.891 & 0.889 \\ [0.5ex]
LoRAHUB (\cite{huang2023lorahub})     & 0.653 & 0.926 & 0.841 & 0.889 & 0.903 & 0.917 & 0.946 & 0.878 & 0.869 \\ [0.5ex]
LoRAMoE (\cite{dou2024loramoe})      & 0.637 & 0.966 & 0.855 & 0.906 & 0.922 & 0.912 & 0.957 & 0.921 & 0.884 \\ [0.5ex]
UniPELT (\cite{mao2022unipelt})     & 0.631 & 0.921 & 0.879 & 0.859 & 0.887 & 0.839 & 0.913 & 0.729 & 0.832 \\ [0.5ex]
\hline
PEML         & 0.698 & 0.967 & 0.848 & 0.905 & 0.892 & 0.964 & 0.952 & 0.899 & 0.891 \\ [0.5ex]
\hline

\end{tabular}
\label{hypernet_glue}
\end{center}
\end{table}

\subsubsection{Comparison to Simpler Alternatives}

\begin{wrapfigure}{r}{0.48\textwidth}
\centering
\scriptsize
\begin{tabular}{ >{\centering\arraybackslash} p{0.32\linewidth} |
                 >{\centering\arraybackslash}p{0.12\linewidth}
                 >{\centering\arraybackslash}p{0.14\linewidth}
                 >{\centering\arraybackslash}p{0.12\linewidth} }

\hline
Method & GLUE & SuperGLUE & MMLU \\
\hline
\hline
LoRA ($r=96$)   & 0.851 & 0.786 & 0.752 \\
LoRA ($r=192$)  & 0.863 & 0.792 & 0.753 \\
PrefixLayer 3   & 0.792 & 0.753 & 0.688 \\
PrefixLayer 6   & 0.798 & 0.764 & 0.674 \\
\hline
PEML (ours)     & \textbf{0.901} & \textbf{0.842} & \textbf{0.803} \\
\hline

\end{tabular}
\captionof{table}{Performance comparison of PEML with high-rank LoRA and manual prefix baselines across GLUE, SuperGLUE, and MMLU benchmarks.}
\label{table:14}
\end{wrapfigure}

Both prior work and our experiments show that simply increasing LoRA rank or manually designing prefix layers is insufficient to match PEML’s performance. The original LoRA study [1, Table 6] demonstrates that on GPT-3 175B for WikiSQL, a minimal rank of (r=1) achieves 73.4\% accuracy, while increasing to (r=64) yields essentially no improvement (73.5\%), illustrating diminishing returns. Similarly, on LLaMa2-7B (Table 3), raising LoRA rank from (r=96) to (r=160) improves average score by only 0.1 points (75.2 → 75.3). Manually increasing prefix depth also provides minimal benefit, e.g., GLUE improves only from 0.792 → 0.798 when prefix layers are increased from 3 to 6. Table~\ref{table:14} compares PEML against high-rank LoRA and manual prefix baselines across GLUE, SuperGLUE, and MMLU, showing that neither strategy achieves PEML’s gains. These results indicate that PEML’s improvements arise from the sophisticated architectures discovered via PrefixNAS, rather than merely scaling parameters.

\subsubsection{Performance Analysis on Low Data Settings}
\label{lowdata}
We also evaluated cross-task knowledge transfer by training on a small sample from all tasks, observing comparable results as shown in Table \ref{table:8} and Table \ref{table:7}. We evaluated the performance of various PEFT techniques, including PreEmbedd, LoRA, and AdaLoRA, along with their combinations with \PM, on GLUE and SuperGLUE benchmarks. PEML achieved the highest average accuracy with LoRA and AdaLoRA on both benchmark in resource-constrained scenarios.

\begin{table}[ht]
\begin{center}
\caption{Performance of various PEFT methods, including PreEmbedd, LoRA, AdaLoRA, and their combinations with PEML, tested on the T5-large model on GLUE benchmark. Results are reported for low data (500 samples from each task) settings.}
\scriptsize
 \begin{tabular} { >{\centering\arraybackslash} p{0.27314\linewidth} |  >{\centering\arraybackslash}p{0.025\linewidth} >{\centering\arraybackslash}p{0.04\linewidth} >{\centering\arraybackslash}p{0.025\linewidth} >{\centering\arraybackslash}p{0.034\linewidth} >{\centering\arraybackslash}p{0.025\linewidth} >{\centering\arraybackslash}p{0.035\linewidth} >{\centering\arraybackslash}p{0.055\linewidth}  | >{\centering\arraybackslash}p{0.038\linewidth} 
 >{\centering\arraybackslash}p{0.055\linewidth}}
 \hline
    Peft Techniques & SST2 &MRPC	&RTE &COLA	&QQP	&WNLI	& QNLI & AVG \\ [1ex]
  \hline
  \hline
    PreEmbedd  & 0.855 & 0.811 & 0.888	& 0.806 & 0.936 & 0.687 & 0.856 & 0.834\\ [0.5ex]

    LoRA  & 0.932 & 0.871 & 0.891 & 0.859 & 0.944 & 0.718 & 0.940 & 0.880\\ [0.5ex]

    LoRA-PreEmbedd  & 0.938 & 0.863 & 0.910 & 0.875 & 0.944  & 0.718 & 0.939 & 0.884\\ [0.5ex]
    \hline
    \PM & 0.952 & 0.932 & 0.942 & 0.890 & 0.942 & 0.750 & 0.952 & 0.908\\ [0.5ex]
    \hline
    \hline
    AdaLoRA  & 0.932 & 0.895 & 0.896 & 0.871 & 0.950 & 0.734 & 0.952 & 0.890\\ [0.5ex] 

    AdaLoRA-PreEmbedd  & 0.916 & 0.891 & 0.893 & 0.846 & 0.945 & 0.750 & 0.946 & 0.884\\ [0.5ex]
    \hline
    \PM-AdaLoRA  & 0.968 & 0.907 & 0.920 & 0.835 & 0.972 & 0.796 & 0.962 & 0.904\\ [0.5ex]

  \hline
\end{tabular}
\label{table:8}
\end{center}
\end{table}

\begin{table}[ht]
\begin{center}
\caption{performance of various PEFT methods, including PreEmbedd, LoRA, AdaLoRA, and their combinations with \PM\, tested on the T5-large model on SuperGLUE tasks. Results are reported for low data (300 samples from each task) efficiency settings, with the \PM\ achieving the highest average performance.}
\scriptsize
 \begin{tabular} { >{\centering\arraybackslash} p{0.274\linewidth} |    >{\centering\arraybackslash}p{0.04\linewidth} >{\centering\arraybackslash}p{0.04\linewidth} >{\centering\arraybackslash}p{0.04\linewidth} >{\centering\arraybackslash}p{0.06\linewidth} >{\centering\arraybackslash}p{0.04\linewidth}   >{\centering\arraybackslash}p{0.055\linewidth} | 
 >{\centering\arraybackslash}p{0.04\linewidth}}
 \hline
    Peft Techniques  & Boolq &RTE	&COPA &MultiRC	&WIC	&WSC & AVG \\ [1ex]
  \hline
  \hline
    PreEmbedd& 0.911 & 0.891 & 0.750 & 0.877 & 0.781 & 0.875 & 0.847\\ [0.5ex]

    LoRA    & 0.919 & 0.935 & 0.750 & 0.889 & 0.810 & 0.732 & 0.839\\ [0.5ex]

    LoRA \& PreEmbedd   & 0.922 & 0.931 & 0.800 & 0.890 & 0.814 & 0.714 & 0.845\\ [0.5ex]
    \hline
    
    \PM   & 0.912 & 0.871 & 0.925 & 0.876 & 0.829 & 0.767 & 0.863\\ [0.5ex]
    \hline
    \hline
    AdaLoRA   & 0.912 & 0.907 & 0.750 & 0.871 & 0.785 & 0.767 & 0.832\\ [0.5ex] 

    AdaLoRA \& PreEmbedd   & 0.897 & 0.895 & 0.725 & 0.878 & 0.791 & 0.767 & 0.826\\ [0.5ex]
    \hline
    \PM-AdaLoRA   & 0.922	&0.943	&0.825	&0.893	&0.840	&0.803	&0.871\\ [0.5ex]

  \hline
\end{tabular}
\label{table:7}
\end{center}
\end{table}

\subsubsection{Performance Analysis on GLM-10B}
As detailed in Table \ref{table:9}, we also benchmarked \PM\ against other PEFT methods on the GLM-10B model. Our approach demonstrates clear superiority by achieving the highest accuracy on SuperGLUE benchmark with an average score of 0.631.
\begin{table}[!ht]
\begin{center}
\caption{Compare the performance accuracy of \PM\ using GLM-10B \cite{du2021glm} model with state-of-the-art multitasking framework. All the results are directly reported from \cite{wang2023C}.}
\scriptsize
 \begin{tabular} { >{\centering\arraybackslash} p{0.16\linewidth} | >{\centering\arraybackslash}p{0.05\linewidth}   >{\centering\arraybackslash}p{0.05\linewidth} >{\centering\arraybackslash}p{0.05\linewidth} >{\centering\arraybackslash}p{0.05\linewidth} >{\centering\arraybackslash}p{0.05\linewidth} >{\centering\arraybackslash}p{0.05\linewidth}     >{\centering\arraybackslash}p{0.05\linewidth} |
>{\centering\arraybackslash}p{0.05\linewidth}}
 \hline
    Peft Techniques & Boolq	&CB	&COPA	&MultiRC	&RTE	&WIC	&WSC	&AVG\\ [1ex]
    \hline
    \hline
    LoRA	& 0.609	& 0.463	& 0.657	& 0.624	& 0.573	& 0.391	& 0.323	& 0.520 \\ [0.8ex]
                              
    MOE-LoRA	& 0.633	& 0.450	 & 0.634	& 0.640	& 0.612	& 0.403	& 0.396 & 0.538 \\ [0.8ex]       
    
    Poly	& 0.646	&  0.521& 0.655	& 0.656	& 0.621	& 0.417	& 0.470 & 0.569 \\ [0.8ex]                        
    
    MHR	& 0.648	& 0.507	& 0.663	& 0.657	& 0.627	& 0.423 & 0.455	& 0.569 \\ [0.8ex] 
    
    C-Poly	& 0.673	& 0.603	& 0.704	& 0.679	& 0.680	& 0.487	& 0.534	& 0.622 \\ [0.8ex] 

    \hline
    \PM   & \textbf{0.682}	& \textbf{0.610}	& \textbf{0.707}	&\textbf{0.688}	&\textbf{0.693}	&\textbf{0.491}& \textbf{0.552}	&\textbf{0.631}\\ [0.8ex] 
    
  \hline
\end{tabular}
\label{table:9}
\end{center}
\end{table}

\subsubsection{NAS Comparison: PrefixNAS vs Two-stage vs Surrogate NAS}

\begin{wrapfigure}{r}{0.36\textwidth}
\vspace{-21pt}
\centering
    \includegraphics[width=0.9\linewidth]{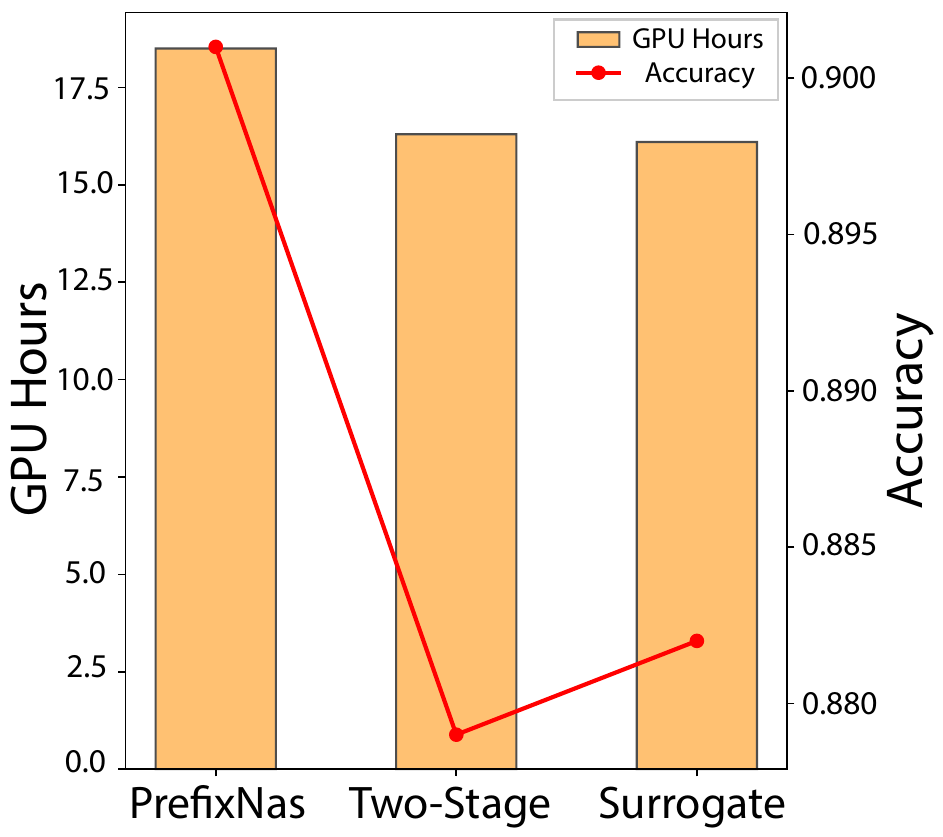}
    \caption{Compute accuracy trade-off for PrefixNAS vs Two-stage vs surrogate-based NAS.}
    \label{fig:7}
    \vspace{-5pt}
\end{wrapfigure}

Exploring two-stage or surrogate-based strategies for PrefixNAS can moderately reduce GPU usage by performing the architecture search on a smaller proxy or using an approximate performance predictor. As shown in Figure \ref{fig:7}, this approach decreases GPU consumption from 18.5 GPU-hours for full PrefixNAS to 16.3 and 16.1 GPU-hours, respectively. However, this computational saving comes with a slight drop in accuracy: full PrefixNAS achieves 0.901, while two-stage and surrogate-based NAS reach 0.879 and 0.882. These results suggest that, although these approximate strategies offer moderate efficiency gains, the full PrefixNAS search remains more effective in identifying optimal prefix architectures, highlighting the trade-off between computation and performance.

\subsubsection{Architecture Transferability of PrefixNAS}
\label{AT}

To evaluate whether PrefixNAS discovers generalizable structures rather than overfitting to specific benchmarks, we tested the transferability of architectures found on GLUE, SuperGLUE, and MMLU to other datasets. We measured performance under three settings: NAS with fine-tuning, no NAS but only fine-tuning, and no NAS without fine-tuning. As shown in Table~\ref{table:13}, direct transfer without fine-tuning results in substantial drops (13–19\%), e.g., GLUE $\rightarrow$ SuperGLUE drops from 0.911 to 0.782, indicating that the architectures encode broadly useful structures. With only fine-tuning, performance remains very close to the NAS-discovered architectures, e.g., GLUE $\rightarrow$ SuperGLUE: 0.911 $\rightarrow$ 0.904, demonstrating that PrefixNAS can identify robust and generalizable prefix architectures that transfer effectively among similar tasks, thereby reducing search costs in practice.

\begin{table}[ht]
\centering
\scriptsize
\caption{Transfer Performance of NAS-Discovered Prefix Architectures Across Datasets. Results confirm that PrefixNAS does not overfit to specific benchmarks. Fine-tuning alone on transferred architectures achieves performance close to full NAS+fine-tuning, while no-fine-tuning drops confirm the search finds broadly useful rather than dataset-specific structures.}
\label{table:13}
\setlength{\tabcolsep}{4pt} 
\renewcommand{\arraystretch}{1.1} 
\begin{tabular}{>{\centering\arraybackslash}p{0.2\linewidth}|
                >{\centering\arraybackslash}p{0.12\linewidth}
                >{\centering\arraybackslash}p{0.12\linewidth}
                >{\centering\arraybackslash}p{0.12\linewidth}}
\hline
\textbf{Transfer Direction} & \textbf{NAS + Fine-tuning} & \textbf{No NAS, Only Fine-tuning} & \textbf{No NAS, No Fine-tuning} \\[1ex]
\hline 
\hline
GLUE $\rightarrow$ SuperGLUE & 0.879 & 0.887 & 0.743 (-15.5\%) \\[1ex]
SuperGLUE $\rightarrow$ GLUE & 0.911 & 0.904 & 0.782 (-14.1\%) \\[1ex]
GLUE $\rightarrow$ CR & 0.830 & 0.821 & 0.718 (-13.5\%) \\[1ex]
SuperGLUE $\rightarrow$ MMLU & 0.803 & 0.798 & 0.691 (-13.9\%) \\[1ex]
MMLU $\rightarrow$ SuperGLUE & 0.879 & 0.877 & 0.713 (-18.9\%) \\[1ex]
MMLU $\rightarrow$ GLUE & 0.911 & 0.912 & 0.776 (-14.9\%) \\[1ex]
GLUE $\rightarrow$ MMLU & 0.803 & 0.816 & 0.684 (-14.9\%) \\[1ex]
\hline
\end{tabular}
\end{table}



\subsubsection{Benchmark Comparison: DeepSeek 7B vs Qwen2-7B vs LLaMA3-8B}

We selected our primary models to ensure fair comparison with prior work, since they are widely used and well-established in the parameter-efficient tuning literature. Most baseline studies rely on these models, making them natural choices for consistent evaluation. While newer models such as LLaMa3-8B \cite{grattafiori2024llama} , DeepSeek-LLM-7B \cite{bi2024deepseek}, and Qwen2-7B \cite{qwen2} show strong potential, we opted for LLaMa2-7B due to its maturity, stable training dynamics, and broad community support. Nonetheless, we also evaluated our PEML approach on several state-of-the-art models across GLUE, SuperGLUE, MMLU, and Commonsense Reasoning (CR) benchmarks in table \ref{table:11}.

\begin{table}[!ht]
\begin{center}
\caption{Benchmark performance of recent LLMs on MMLU, GLUE, SuperGLUE, and Commonsense Reasoning (CR). Average accuracy across all benchmarks is reported for each model.}
\scriptsize
 \begin{tabular} { >{\centering\arraybackslash} p{0.22\linewidth} | >{\centering\arraybackslash}p{0.07\linewidth}   >{\centering\arraybackslash}p{0.07\linewidth} >{\centering\arraybackslash}p{0.07\linewidth} >{\centering\arraybackslash}p{0.07\linewidth} }
 \hline
    SOTA Models & MMLU & GLUE & SuperGLUE & CR \\ [1ex]
    \hline
    \hline
    LLaMA3-8B   & 0.662  & 0.922 & 0.915  & 0.821  \\ [0.8ex]
    DeepSeek-7B  & 0.569 & 0.910 & 0.897 & 0.834 \\ [0.8ex]
    Qwen2-7B   & 0.694  & 0.904 & 0.882 & 0.809  \\ [0.8ex]

  \hline
\end{tabular}
\label{table:11}
\end{center}
\end{table}

\subsubsection{Heterogeneous Multi-Task Evaluation}

To evaluate PEML's performance across fundamentally different task types, we conducted a heterogeneous multi-task experiment jointly training LLaMA2-7B on three distinct tasks: SST-2 (sentiment classification), GSM8K (mathematical reasoning) \cite{cobbe2021training}, and CNN/DailyMail (text summarization) \cite{see-etal-2017-get}, with each task contributing equally during training through round-robin sampling from their respective datasets. We employed accuracy for SST-2 and GSM8K, and ROUGE-L for CNN/DailyMail summarization. For evaluation, SST-2 used its standard validation set, GSM8K employed 5-shot prompting on the test set, and CNN/DailyMail measured ROUGE-L on the validation set with beam search (beam=4). 

\begin{table}[ht]
\centering
\scriptsize
\caption{Performance on Heterogeneous Task Mixture (LLaMA2-7B)}
\label{tab:heterogeneous}
\vspace{0.3cm}
\begin{tabular}{c|ccc}
\hline
\textbf{Method} & \textbf{SST-2 (Acc) ↑} & \textbf{GSM8K (Acc) ↑} & \textbf{CNN/DM (R-L) ↑}  \\ [0.8ex]
\hline
LoRA & 0.892 & 0.281 & 0.243  \\ [0.8ex]
PreEmbedd & 0.845 & 0.267 & 0.229  \\ [0.8ex]
PEML (ours) & \textbf{0.904} & \textbf{0.288} & \textbf{0.268} \\ [0.8ex]
\hline
\end{tabular}
\end{table}

Architectural analysis revealed that PrefixNAS learned distinct computational patterns for each task type: for SST-2, it predominantly selected ReLU activation with moderate dropout (0.3); for GSM8K, it favored Tanh activation with layer normalization, potentially stabilizing numerical reasoning pathways; and for CNN/DailyMail, it frequently chose GELU activation with skip connections. All tasks were trained simultaneously with shared parameters except for task-specific LoRA matrices and PrefixNAS-optimized continuous prompts, enabling knowledge transfer while maintaining task-specific adaptation through our unified architectural framework, with the learned architectures demonstrating PEML's ability to discover specialized computational structures beyond simple prompt tuning.

\subsubsection{Comparison of Sequential and Parallel Optimization}
\label{SP}
We further investigate the effect of combining LoRA and Prefix Tuning under different optimization approaches. Table \ref{table:10} reports the performance on five representative SuperGLUE tasks. We consider three settings: (i) sequential optimization where LoRA is trained first followed by Prefix Tuning, (ii) sequential optimization in reverse order where Prefix Tuning is trained first followed by LoRA, and (iii) parallel optimization where LoRA and Prefix Tuning are jointly optimized (\PM). Results demonstrate that parallel optimization consistently yields higher average performance, while sequential variants show task-specific variations.

\begin{table}[ht]
\begin{center}
\caption{Performance comparison of sequential and parallel optimization of LoRA and Prefix Tuning across SuperGLUE tasks. Parallel optimization achieves the highest average performance.}
\scriptsize
 \begin{tabular} { >{\centering\arraybackslash} p{0.22\linewidth} | >{\centering\arraybackslash}p{0.05\linewidth}   >{\centering\arraybackslash}p{0.05\linewidth} >{\centering\arraybackslash}p{0.05\linewidth} >{\centering\arraybackslash}p{0.05\linewidth} >{\centering\arraybackslash}p{0.05\linewidth}  |
>{\centering\arraybackslash}p{0.05\linewidth}}
 \hline
    Approches & BoolQ & COPA & RTE & WSC & WiC & Avg \\ [1ex]
    \hline
    \hline
    LoRA $\rightarrow$ Prefix            & 0.929 & 0.855 & 0.922 & 0.750 & 0.835 & 0.858 \\ [0.8ex]
    Prefix $\rightarrow$ LoRA            & 0.922 & 0.835 & 0.938 & 0.735 & 0.864 & 0.859 \\ [0.8ex]
    LoRA $\parallel$ Prefix (\PM)   & 0.925 & 0.850 & 0.932 & 0.804 & 0.837 & \textbf{0.869} \\ [0.8ex]
  \hline
\end{tabular}
\label{table:10}
\end{center}
\end{table}

\subsection{Ablation Study}

We conduct ablation studies to evaluate the impact of three critical aspects in \PM: the optimization order in \PM\, the number of layers (n), and the search space operations within PrefixNAS. For each experiment, we use T5-large model and train it on the SuperGLUE dataset. (1) We observe that parallel optimization consistently outperforms sequential optimization whether LoRA is followed by PrefixNAS or vice versa. Moreover, sequential order has more additional training overhead than the parallel order ($\sim$ see Appendix \ref{SP}). (2)  We increase the number of layers up to 8 but beyond n=6 does not yield a significant performance boost ($\sim$ see Figure \ref{fig3}). However, A more complex architecture with additional layers could theoretically improve performance but it also introduces a large number of trainable parameters. Therefore, we must find a balance between architectural complexity and the associated computational cost. Lastly, (3) We find that certain operations did not have any major contribution to the PrefixNAS architecture. We exclude these operations during our search space design ($\sim$ see Figure \ref{fig3}). This reduction in search space results in significant time savings during the search process.

\subsection{Inference Latency Induced by Adapter Switching}

Inference latency in multi-task setting is affected by PEFT adapter switching. For 100 tasks with separate adapters, each task's latency is the forward pass $t_f$ plus switch time $t_s$, giving total latency $T = 100 \times (t_f + t_s)$. In PEML, a single unified adapter removes switching ($t_s = 0$), so $T_{\text{PEML}} = 100 \times t_f$. Using empirical estimates with minor variability, T5-large has $t_f = 11$ ms, average $t_s \approx 2.1$ ms, resulting in 1,320 ms for 100 adapters vs 1,100 ms with PEML ($\sim 17\%$ reduction); LLaMA2-7B has $t_f = 52$ ms, average $t_s \approx 4.3$ ms, giving 5,630 ms vs 5,200 ms ($\sim 8\%$ reduction). The switching overhead scales linearly with the number of tasks, so PEML's unified-adapter design becomes increasingly advantageous for larger multi-task setups, maintaining efficiency without compromising forward-pass computation.

\subsection{\PM\ vs. MultiLoRA: VRAM-Efficient Scaling}
\label{vram}
Training throughput, and VRAM usage are critical for generative LLMs. MultiLoRA increases model capacity by adding multiple parallel LoRA modules ($n>1$), which linearly increases VRAM usage due to activation caching, especially for long sequences. For example, training LLaMA-7B with sequences of 1024 tokens and  n=5 modules can consume more VRAM than full-parameter fine-tuning, limiting practical scalability. In contrast, our proposed PEML maintains a single LoRA module (n=1) and increases the rank r to match or exceed the expressivity of MultiLoRA. This design avoids horizontal scaling entirely and VRAM usage remains nearly constant regardless of the rank increase. Additionally, MultiLoRA and \PM\ do not introduce notable latency and the throughput remains close to around 400 tokens per GPU per second ($\sim$ see figure \ref{fig4}). In our benchmarking, the throughput of \PM\ is almost twice that of full parameter fine-tuning (208 tokens per GPU per second)

\begin{figure*}[!ht]
\centering
    \includegraphics[width=0.8\textwidth]{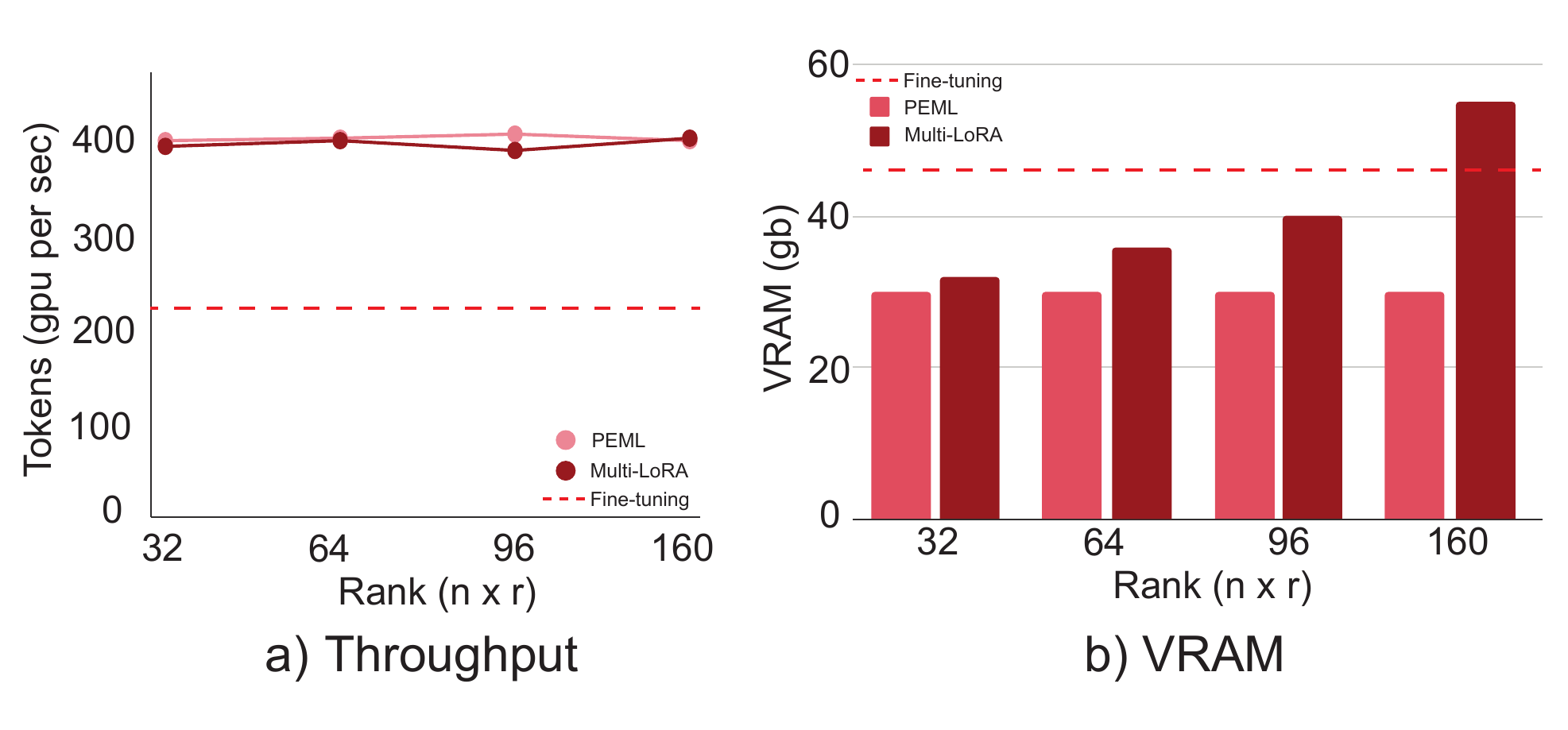}
    \caption{ (a) Throughput and (b) peak VRAM usage benchmarked when training LLaMA-7B with sequences of 1024
tokens. n × r on horizontal axis indicates total rank of MultiLoRA and \PM.}
    \label{fig4}
\end{figure*}

\subsection{Stability of the Differentiable–Discrete Transition}

PrefixNAS conducts architecture search using a continuous relaxation and finalizes the architecture via an argmax operation \ref{eq:7}. To evaluate the stability of this transition, we compare Softmax + Argmax with Gumbel-Softmax \cite{liu2023bridging} and Straight-Through Estimators (STE) \cite{huijben2022} in terms of gradient norm, architectural parameter variance, and performance consistency after discretization. As shown in Table \ref{table:12}, our method yields the lowest gradient norm (0.21 ± 0.05) and parameter variance (0.014 ± 0.006), indicating stable gradient-based optimization during search. We further measure the discretization gap, defined as the performance drop when moving from the relaxed to the final discrete architecture. Softmax + Argmax exhibits a small gap of 0.7, demonstrating that the searched architecture remains stable after discretization. While Gumbel-Softmax shows a slightly smaller gap (0.6), it results in lower final accuracy, and STE exhibits both higher gradient noise and a substantially larger gap (1.1).

\begin{figure*}[ht]
\centering
    \includegraphics[width=0.99\textwidth]{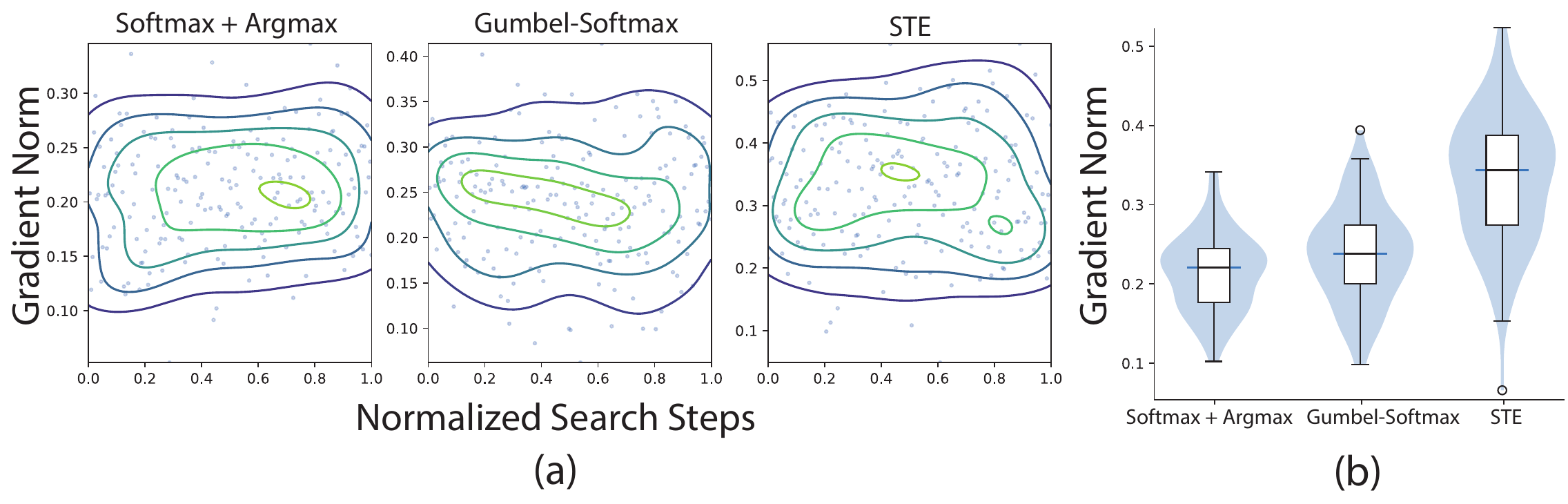}
    \caption{ (a) KDE contours of architectural gradient norms during search. Softmax + Argmax shows tighter concentration, indicating stable optimization. and (b) Distribution of architectural gradient norms during search. Softmax + Argmax exhibits tighter variance, indicating more stable optimization.}
    \label{fig5}
\end{figure*}

\label{DDT}
\begin{table}[ht]
\centering
\scriptsize
\caption{Comparison of relaxation techniques for architecture search. Lower gradient norm, architectural parameter variance ($\theta$ Var), and discretization gap indicate more stable optimization and reliable transition from continuous to discrete architectures.}
\label{table:12}
\setlength{\tabcolsep}{4pt} 
\renewcommand{\arraystretch}{1.1} 
\begin{tabular}{>{\centering\arraybackslash}p{0.2\linewidth}|
                >{\centering\arraybackslash}p{0.1\linewidth}
                >{\centering\arraybackslash}p{0.13\linewidth}
                >{\centering\arraybackslash}p{0.14\linewidth}
                >{\centering\arraybackslash}p{0.15\linewidth}}
\hline
Method & Grad Norm $\downarrow$ & $\theta$ Var $\downarrow$ & GLUE Avg (\%) $\uparrow$ & Discretization Gap $\downarrow$ \\[1ex]
\hline
\hline
Softmax + Argmax & \textbf{0.21 $\pm$ 0.05} & \textbf{0.014 $\pm$ 0.006} & \textbf{90.1} & 0.7 \\ [1ex]
Gumbel-Softmax   & 0.24 $\pm$ 0.06          & 0.017 $\pm$ 0.007          & 89.2          & \textbf{0.6} \\ [1ex]
STE              & 0.33 $\pm$ 0.09          & 0.031 $\pm$ 0.011          & 88.5          & 1.1 \\
\hline
\end{tabular}
\end{table}

In addition, we provide a box plot \ref{fig5} of architectural gradient norms to illustrate the distributional stability during search and a Kernel Density Estimation (KDE) contour plot \ref{fig5} visualizes the temporal concentration of gradients across search steps, further confirming stable and consistent optimization behavior.

\subsection{LLM Usage}

We used a large language model (LLM) solely for writing polish, such as improving grammar, clarity, and readability of the text. The LLM did not contribute to research ideation, methodology, analysis, or results. All scientific content and conclusions are the responsibility of the authors.

\newpage

\end{document}